%% file: arxiv.tex
\documentclass[11pt]{article}

\usepackage{fullpage}

\usepackage[utf8]{inputenc} 
\usepackage[T1]{fontenc}    
\usepackage{hyperref}       
\usepackage{url}            
\usepackage{booktabs}       
\usepackage{amsfonts}       
\usepackage{nicefrac}       
\usepackage{microtype}      

\usepackage{amsmath, amsthm}
\usepackage{wrapfig}
\usepackage{thmtools}
\usepackage{thm-restate}

\usepackage{enumitem}
\usepackage[]{algorithm}
\usepackage{algorithmic}

\usepackage[numbers]{natbib}

\title{Upper Confidence Bounds for Combining Stochastic Bandits}

\author{
   
   Ashok Cutkosky \\
   Boston University \\
   Boston, MA \\
   \texttt{ashok@cutkosky.com} 
   \\
   \and
   Abhimanyu Das \\
   Google Research \\
   Mountain View, CA \\
   \texttt{abhidas@google.com}
   \\
   \and
   Manish Purohit \\
   Google Research \\
   Mountain View, CA \\
   \texttt{mpurohit@google.com}
}

\input{header.tex}

\newif\ifarxiv
\arxivtrue

\date{}
\begin{document}

\maketitle

\begin{abstract}
  We provide a simple method to combine stochastic bandit algorithms. Our approach is based on a ``meta-UCB'' procedure that treats each of $N$ individual bandit algorithms as arms in a higher-level $N$-armed bandit problem that we solve with a variant of the classic UCB algorithm. Our final regret depends only on the regret of the base algorithm with the best regret in hindsight. 
  This approach provides an easy and intuitive alternative strategy to the CORRAL algorithm for adversarial bandits, without requiring the  stability conditions imposed by CORRAL on the base algorithms. 
  Our results match lower bounds in several settings, and we provide empirical validation of our algorithm on misspecified linear bandit and model selection problems.
\end{abstract}

\input{intro}

\input{setup}

\input{mainalg}

\input{gapbound}

\input{doubling}

\input{examples}

\input{experiments}

\input{conclusion}

\bibliographystyle{unsrt}
\bibliography{references}

\appendix
\onecolumn
\clearpage
\input{mainalgproof}

\input{gapboundproof}
\input{doublingproof}

\input{examplesproofs}
\input{ucbbounds}
\input{adv_context}
\input{adv_linucb}
\end{document}

%% file: header.tex
\usepackage{color}
\usepackage{amsmath,amscd,amssymb,amsthm}
\usepackage{verbatim}
\usepackage{thmtools}
\usepackage{thm-restate}

\input{symbol}

\usepackage{algorithm}
\usepackage{algorithmic}

\usepackage{multicol}
\declaretheorem[name=Theorem]{Theorem}

\declaretheorem[name=Lemma, numberlike=Theorem]{Lemma}

\declaretheorem[sibling=Theorem]{Corollary}

\usepackage{graphicx}
\usepackage{caption}
\usepackage{subcaption}
\usepackage{wrapfig}

\newcommand{\regret}{\text{Regret}}

%% file: symbol.tex
\newcommand{\A}{\mathcal{A}}

\newcommand{\R}{\mathbb{R}}
\newcommand{\E}{\mathop{\mathbb{E}}}
\newcommand{\argmin}{\mathop{\text{argmin}}}
\newcommand{\argmax}{\mathop{\text{argmax}}}

%% file: intro.tex
\section{Introduction}

This paper studies the classic \emph{contextual bandit problem} in a stochastic setting \cite{langford2008epoch, beygelzimer2011contextual}, which is a generalization of the even more classical multi-armed bandit problem \cite{lai1985asymptotically}. In each of $T$ rounds indexed by $t=1,\dots,T$, we observe an i.i.d. random \emph{context} $c_t\in C$, which we use to select an \emph{action} $a_t=x_t(c_t)\in A$ based on some \emph{policy} $x_t:C\to A$. Then, we receive a noisy reward $\hat r_t\in[0,1]$, whose expectation is a function of only $c_t$ and $a_t$: $\E[r_t|c_t,a_t] = r(c_t, a_t)$. The goal is to perform nearly as well as the best policy in hindsight by minimizing the \emph{regret}:
\begin{align*}
    \regret = \sum_{t=1}^T r_\star - r(c_t, a_t)
\end{align*}
where $r_\star=\inf_{x\in X} \E_{c}[r(c, x(c))]$, and $X$ is some space of possible policies.

This problem and variants has been extensively studied under diverse assumptions about the space of policies $X$ and distributions of the rewards and values for $r_\star$. (e.g. see \cite{auer2002finite, langford2008epoch, beygelzimer2011contextual, lattimore2018bandit, tewari2017ads, agarwal2014taming}).  Many of these algorithms have different behaviors in different environments (e.g. one algorithm might do much better if the reward $r(c_t,a_t)$ is a linear function of the context, while another might do better if the reward is independent of the context). This plethora of prior algorithms necessitates a ``meta-decision'': If the environment is not known in advance, which algorithm should be used for the task at hand? Even in hindsight, it may not be obvious which algorithm was most optimized for the experienced environment, and so this meta-decision can be quite difficult.

We model this meta-decision by assuming we have access to $N$ \emph{base bandit algorithms} $\A_1,\dots,\A_N$. We will attempt to design a meta-algorithm whose regret is comparable to the best regret experienced by any base algorithm in hindsight for the current environment. Since we don't know in advance which base algorithm will be optimal for the current environment, we need to address this problem in an online fashion. On the $t$th round, we will choose some index $i_t$ and play the action suggested by the algorithm $\A_{i_t}$.   The primary difficulty is that some base algorithms might perform poorly at first, and then begin to perform well later. A naive strategy might discard such algorithms based on the poor early performance, so some enhanced form of exploration is necessary for success. A pioneering prior work on this setting has considered the \emph{adversarial} rather than stochastic case \citep{agarwal2017corralling}, and utilizes a sampling based on mirror descent with a clever mirror-map. Somewhat simplifying these results, suppose each algorithm $\A_i$ guarantees regret $C_iT^{\alpha_i}$ for some $C_i$ and $\alpha_i$\footnote{in many common settings, $\alpha_i=1/2$}. Given a user-specified learning rate parameter $\eta\in \R$, \cite{agarwal2017corralling, pacchiano2020model} guarantee:
\begin{align}
    \regret \le \min_{j\le N}   C_j^{\frac{1}{\alpha_j}}T\eta ^{\frac{1-\alpha_j}{\alpha_j}} +T\eta+ \frac{N}{\eta}\label{eqn:corral}
\end{align}
The value of $\eta$ is chosen apriori, but the values of $C_i$ need not be known. To gain a little intuition for this expression, suppose all $\alpha_i=1/2$, and set $\eta = \frac{\sqrt{N}}{\sqrt{T}}$. Then the regret is $\min_{i \le N} C_i^2\sqrt{NT}$.

In this paper, we leverage the stochastic environment to avoid requiring some technical stability conditions needed in \cite{agarwal2017corralling}: our result is a true black-box meta-algorithm that does not require any modifications to the base algorithms. Moreover, our general technique is in our view both different and much simpler. Our regret bound also improves on (\ref{eqn:corral}) by virtue of being non-uniform over the base algorithms: given any parameters $\eta_1,\dots,\eta_N$, we obtain
\begin{align}
    \regret \le \min_{j\le N} C_j^{\frac{1}{\alpha_j}}T\eta_j ^{\frac{1-\alpha_j}{\alpha_j}} + T\eta_j+\sum_{i\ne j} \frac{1}{\eta_i}\label{eqn:informal}
\end{align}
This recovers (\ref{eqn:corral}) when all $\eta_i$ are equal. In general one can think of the values $\frac{1}{\eta_i}$ as specifying a kind of \emph{prior} over the $\A_i$ which allows us to develop more delicate trade-offs between their performances. For example, consider again the setting when all $\alpha_i=1/2$. If we believe that $\A_1$ is more likely to perform well, then by setting $\eta_1 = \frac{1}{\sqrt{T}}$ and $\eta_j = \frac{N}{\sqrt{T}}$ for $j \neq 1$, we can obtain a regret of $\min(C_1^2 \sqrt{T}, \min_{2 \le j \le N} C_j^2 N \sqrt{T}))$.

This type of bound is in some sense a continuation of a general trend in the bandit community towards finding algorithms that adapt to various properties of the environment that are unknown in foresight (e.g \cite{bubeck2012best,wei2018more, audibert2009exploration,ghosh2020}). However, instead of committing to some property (e.g. large value of $r_\star$, or small variance of $\hat r$), we instead design an algorithm that is in some sense ``future-proof'', as new algorithms can be easily incorporated as new base algorithms $\A_i$. 

In a recent independent work, \cite{pacchiano2020model} has extended the techniques of \cite{agarwal2017corralling} to our same stochastic setting. They use a clever smoothing technique to also dispense with the stability condition required by \cite{agarwal2017corralling}, and achieve the regret bound (\ref{eqn:corral}). In addition to achieving the non-uniform bound (\ref{eqn:informal}), we also improve upon their work in two other ways: our algorithm requires only $O(N)$ space in contrast to $O(TN)$ space, and we allow for our base algorithms to only guarantee in-expectation rather than high-probability bounds.

In the stochastic setting, it is frequently possible to obtain \emph{logarithmic} regret subject to various forms of ``gap'' assumptions on the actions. However, the method of \cite{agarwal2017corralling}, even when considered in the stochastic setting as in \cite{pacchiano2020model}, seems unable to obtain such bounds. Instead, \cite{arora2020corralling} has recently provided a method based on UCB that can achieve such results. Their algorithm is similar to ours, but we devise a somewhat more intricate method that is able to not only obtain the results outlined previously, but also match the logarithmic regret bounds provided by \cite{arora2020corralling}.

In the stochastic setting, \cite{yadkori2020} introduces a new model selection technique called Regret Balancing. However this approach requires knowledge of the exact regret bounds of the optimal base algorithm. Our approach not only avoids this requirement, but also results in stronger regret guarantees than those in that paper.

We also consider an extension of our techniques to the setting of linear contextual bandits with \emph{adversarial} features. In this case we require some modifications to our combiner algorithm and assume that the base algorithms are in fact instances of linUCB \cite{chu2011contextual} or similar confidence-ellipsoid based algorithms. However, subject to these restrictions we are able to recover essentially similar results as in the non-adversarial setting, which we present in Section \ref{sec:advcombiner}.

The rest of this paper is organized as follows. In section \ref{sec:setup} we describe our formal setup and some assumptions. In section \ref{sec:mainalg} we provide our algorithm and main regret bound in a high-probability setting. In section \ref{sec:doubling}, we extend our analysis to in-expectation regret bounds as well as providing an automatic tuning of some parameters in the main algorithm. In section \ref{sec:examples} we sketch some ways in which our algorithm can be applied, and show that it matches some prior lower bound frontiers. Finally, in section \ref{sec:experiments} we provide empirical validation of our algorithm.

%% file: setup.tex
\section{Problem Setup}\label{sec:setup}
We let $A$ be a space of actions, $X$ a space of policies, and $C$ a space of contexts. Each policy is a function $C\to A$. (random policies can be modeled by pushing the random bits into the context). In each round $t=1,\dots,T$ we choose $x_t\in X$, then see an i.i.d. random context $c_t$, and then receive a random reward $\hat r_t\in[0,1]$. Let $H_t$ denote the sequence $x_1,c_1,\hat r_1,\dots,\hat r_{t-1}, x_t, c_t$. There is an unknown function $r:C\times A\to \R$ such that $\E[\hat r_t|H_t]=r(c_t,x_t(c_t))$ for any $H_t$. The distribution of $\hat r_t$ is independent of all other values conditioned on $x_t$ and $c_t$. 
We will also write $r(x)=\E_{c\sim \mathcal{D}_c}[r(c,x)]$ and $r_t = r(c_t, x_t(c_t))$.
Let $x_\star \in \argmax r(x)$. Define  $r_\star = \E_{c}r(c, x_\star(c))$. Then we define the regret (often called ``pseudo-regret'' instead) as:
\begin{align*}
    \regret = \sum_{t=1}^T r_\star - r_t
\end{align*}

Each base bandit algorithm $\A_i$ can be viewed as a randomized procedure that takes any sequence $x_1,c_1,\hat r_1,\dots,x_{t-1}, c_{t-1}, \hat r_{t-1}$ and outputs some $x_t\in X$. At the $t$th round of the bandit game, our algorithm will choose some index $i_t\in \{1,\dots,N\}$. Then we obtain policy $x^{i_t}_{T(i_t, t)}$ from $A_{i_t}$, and take action $x_t=x^{i_t}_{T(i_t,t)}(c_t)$. The policy $x^{i_t}_{T(i_t,t)}$ is the output of $A_{i_t}$ on the input sequence of policies, contexts, and rewards for all the prior rounds for which we have chosen this same index $i_t$. After receiving the reward $\hat r_t$, we send this reward as feedback to $\A_{i_t}$. 

In order to formalize this analysis more cleanly, for all $i$, we define $T$ independent random variables $c^i_1,\dots,c^i_T$, where each $c^i_t\sim \mathcal{D}_c$. Further, we define random variables $x^i_1,\dots,x^i_T$ and $\hat r^i_1,\dots,\hat r^i_T$ such that $x^i_t$ is the output of $\A_i$ on the input sequence $x^i_1,c^i_1,\hat r^i_1,\dots,x^i_{t-1}, c^i_{t-1}, \hat r^i_{t-1}$ and $\hat r^i_t$ is a reward obtained by choosing policy $x^i_t$ with context $c^i_t$. We also define random variables $r^i_t=\E[\hat r^i_t|x^i_1,c^i_1,\hat r^i_1,\dots,\hat r^i_{t-1},c^i_{t-1}]$.

Then we can rephrase the high-level action of our algorithm as follows: We choose some index $i_t\in \{1,\dots,N\}$. Then we play policy $x_t=x^{i_t}_{T(i_t, t)}$, see context $c_t=c^{i_t}_{T(i_t,t)}$, and obtain reward $\hat r_t=\hat r^{i_t}_{T(i_t,t)}$. We define $r_t =r^{i_t}_{T(i_t,t)}$. Note that the distribution of the observed reward $\sum_{t=1}^T \hat r_t$ as well as the expected reward $\sum_{t=1}^T r_t$ is maintained by this description of the random variables.

\subsection{Assumptions}
We will consider two settings for the base algorithms in our analysis. First, a high-probability setting in which we wish to provide a regret bound that holds with probability at least $1-\delta$ for some given $\delta$. Second, an in-expectation setting in which we simply wish to bound the expected value of the regret. In the first setting, we require high-probability bounds on the regret of the base algorithms, while in the second setting we do not. Our approach to the in-expectation setting will be a construction that uses our high-probability algorithm as a black-box. Thus, the majority of our analysis takes place in the high-probability setting (in Section \ref{sec:mainalg}), for which we now describe the assumptions.

In the high-probability setting, we assume there are (known) numbers $C_1,\dots,C_N$ and $\alpha_1,\dots,\alpha_N\in[\frac{1}{2},1]$, a $\delta>0$ and an (unknown) set $S\subset \{1,\dots,N\}$ such that with probability at least $1-\delta$, there is some $J\in S$ such that for all $t\le T$,
\begin{align}
    \sum_{\tau=1}^t r_\star - r^J_\tau\le C_Jt^{\alpha_J}\label{eqn:regret_requirement}
\end{align}
Intuitively, this is saying that each algorithm $\A_i$ comes with a \emph{putative} regret bound $C_it^{\alpha_i}$, and with high probability there is \emph{some} $\A_i$ with $i\in S$ for which its claimed regret bound is in fact correct. The assumption is stronger as $S$ becomes smaller, and our final results will depend on the size of $S$. In Section \ref{sec:ucb}, we provide some examples of algorithms that satisfy the requirement (\ref{eqn:regret_requirement}). Generally, it turns out that most algorithms based the optimism principle can be made to work in this setting for any desired $\delta$. We will refer to an index $J$ that satisfied (\ref{eqn:regret_requirement}) as being ``well-specified''.

To gain intuition about our results, we recommend that the reader supposes that $S$ is a singleton $S=\{J\}$ for some unknown index $J$, and $\alpha_i=1/2$ for all $i$ (note that $S$ may be much smaller than the set of indices for which \ref{eqn:regret_requirement} holds). As a concrete example, suppose $|X|=K$ and $c_t=0$ for all $t$, so that we are playing a classic $K$-armed bandits problem. Let $\A_1$ be an instance of UCB, but restricted to the first $10$ arms, while $\A_2$ is instance of UCB restricted to the last $K - 10$ arms. Then, we can set $C_1=O(\sqrt{10\log(T^2/\delta)}$ and $C_2=O(\sqrt{(K-10)\log(T^2/\delta)}$. Now, depending on which is the optimal arm, we may have $S=\{1\}$ or $S=\{2\}$. Although $S=\{1,2\}$ would also satisfy the assumptions, since our results improve when $S$ is smaller and $S$ is unknown to the algorithm, we are free to choose the smallest possible $S$.


These assumptions may seem stronger than prior work at first glance: not only do we require high-probability rather than in-expectation regret bounds, we require at least one base algorithm to be well-specified, and we require knowledge of the putative regret bounds through the coefficients $C_i$ and $\alpha_i$. Prior work in this setting (e.g. \citep{agarwal2017corralling, pacchiano2020model}) dispenses with the last two requirements, and \citep{agarwal2017corralling} also dispenses with the high-probability requirement. However, it turns out that through a simple doubling construction, we can easily incorporate base algorithms with unknown regret bounds that hold only in expectation into our framework. The ability to handle unknown regret bounds dispenses with the well-specified assumption. We describe this construction and the relevant assumptions in Section \ref{sec:doubling}, and show that it only weakens our analysis by log factors.


%% file: mainalg.tex
\section{UCB over Bandits}
\label{sec:mainalg}
In this Section, we describe our meta-learner for the high-probability setting. The intuition is based on upper-confidence bounds: first, we observe that the unknown well-specified algorithm $\A_J$'s rewards $\hat r^J_\tau$ behave very similarly to independent bounded random variables with mean $r_\star$ in that their average value $\hat \mu^J_t =\frac{1}{t}\sum_{\tau=1}^t \hat r^J_\tau$ concentrates about $r_\star$ with radius $C_Jt^{\alpha_J-1}$. From this, one might imagine that for any index $i$, the value of $\hat \mu^i_t + C_it^{\alpha_i-1}$ gives some kind of upper-confidence bound for the final total reward of algorithm $i$. We could then feed these estimates into a UCB-style algorithm that views the $N$ base algorithms as $N$ arms. Unfortunately, such an approach is complicated by two issues. First, the putative regret bounds for each $\A_i$ may not actually hold, which could damage the validity of our confidence estimates. Second, the confidence bounds for different algorithms may have very imbalanced behavior due to the different values of $C_i$ and $\alpha_i$, and we would like our final regret bound to depend only on $C_J$ and $\alpha_J$.

We address the first issue by keeping track of the statistic $\sum_{\tau=1}^t \hat \mu^i_\tau - \hat r^i_\tau$. If $\sum_{\tau=1}^t \hat \mu^i_\tau - \hat r^i_\tau\ge C_it^{\alpha_i}$ at any time, then we can conclude that $i$ is not the well-specified index $J$ and so we simply discard $\A_i$. Moreover, it turns out that so long as $\sum_{\tau=1}^t \hat \mu^i_\tau - \hat r^i_\tau\le C_it^{\alpha_i}$, the rewards $\hat r^i_\tau$ are ``well-behaved'' enough that our meta-UCB algorithm can operate correctly.

We address the second issue by employing \emph{shifted confidence intervals} in a construction analogous to that employed by \cite{lattimore2015pareto}, who designed a $K$-armed bandit algorithm with a regret bound that depends on the identity of the best arm. Essentially, each algorithm $\A_i$ is associated with a \emph{target regret bound}, $R_i$, and each confidence interval is decreased by $\frac{R_i}{T}$. Assuming $R_i$ satisfies some technical conditions, this will guarantee that the regret is at most $O(R_J)$ for any $J$ such that $\A_J$ is well-specified.

Formally, our algorithm is provided in Algorithm \ref{alg:combiner}, and its analysis is given in Theorem \ref{thm:combiner}, proved in Appendix \ref{sec:mainalgproof}. Note that little effort has been taken to improve the constants or log factors.
\begin{algorithm*}
\caption{Bandit Combiner}\label{alg:combiner}
\begin{algorithmic}
   \STATE{\bfseries Input: } Bandit algorithms $\A_1,\dots,\A_N$, numbers $C_1,\dots,C_N$, $\alpha_1,\dots,\alpha_N$, $R_1,\dots,R_N$, $T$.
   \STATE Set $T(i,0)=0$ for all $i$, set $\hat\mu^i_0=0$ for all $i$, and set $I_1 = \{1,\dots,N\}$
   \FOR{$t=1\dots T$}
   \STATE Set $U(i,t-1)=\hat \mu^i_{T(i,t-1)} + \min\left(1,\frac{C_iT(i,t-1)^{\alpha_i} + \sqrt{8\log(T^3N/\delta)T(i,t-1)}}{T(i,t-1)}\right) -\frac{R_i}{T}$ for all $i$.
   \STATE Set $i_t = \argmax_{i \in I_t}U(i,t-1)$.
   \STATE Update $T(i_t,t)=T(i_t,t-1)+1$ and $T(j, t)=T(j,t-1)$ for $j\ne i_t$.
   \STATE Get $T(i_t,t)$th policy $x_t=x^{i_t}_{T(i_t, t)}$ from $A_{i_t}$. See context $c_t$ and play action $x_t(c_t)$.
   \STATE Receive reward $\hat r_t = \hat r^i_{T(i_t, t)}$, provide reward $\hat r_t$ and context $c_t$ as feedback to $\A_{i_t}$.
   \STATE Update $\hat \mu^{i_t}_{T(i_t,t)} = \frac{1}{T(i_t,t)}\sum_{\tau=1}^{T(i_t,t)}\hat r^{i_t}_\tau$.
   \IF{$\sum_{\tau=1}^{T(i_t,t)} \hat \mu^{i_t}_{\tau-1} - \hat r^{i_t}_\tau\ge C_{i_t}T(i_t, t)^{\alpha_{i_t}} + 3\sqrt{\log(T^3N/\delta)T(i_t,t)}$}
   \STATE $I_t = I_{t-1}-\{i_t\}$.
   \ELSE
   \STATE $I_t = I_{t-1}$.
   \ENDIF
   \ENDFOR
\end{algorithmic}
\end{algorithm*}
\begin{restatable}{Theorem}{thmcombiner}\label{thm:combiner}
Suppose there is a set $S\subset \{1,\dots,N\}$ such that with probability at least $1-\delta$, there is some $J\in S$ such that
\begin{align*}
    \sum_{\tau=1}^t r_\star - r^J_\tau \le C_Jt^{\alpha_J}
\end{align*}
for all $t\le T$. Further, suppose the $C_i$ and $\alpha_i$ are known, and the $R_i$ satisfy:
\begin{align*}
    R_i&\ge C_iT^{\alpha_i} \\
    R_i&\ge \sum_{k\ne i}  \max\left[\frac{(1-\alpha_k)(1+\alpha_k)^{\frac{1}{1-\alpha_k}} (2C_k)^{\frac{1}{1-\alpha_k}}T^{\frac{\alpha_k}{1-\alpha_k}} }{\alpha_k R_k^{\frac{\alpha_k}{1-\alpha_k}}},\right.\\
    &\left.\qquad\qquad\qquad\qquad\frac{288\log(T^3N/\delta) T}{R_k}\right]
\end{align*}
Let $r_t= \E[\hat r^{i_t}_{T(i_t,t)}]$ be the expected reward of Algorithm \ref{alg:combiner} at time $t$. Then, with probability at least $1-3\delta$, the regret satisfies:
\begin{align*}
    \sum_{t=1}^T r_\star - r_t \le 3\sup_{j\in S} R_j
\end{align*}
Note that the algorithm does not know the set $S$.
\end{restatable}

The conditions on $R_i$ in this Theorem are somewhat opaque, so to unpack this a bit we provide the following corollary:
\begin{restatable}{Corollary}{thmcombinereta}\label{thm:combinereta}
Suppose there is a set $S\subset \{1,\dots,N\}$ such that with probability at least $1-\delta$, there is some $J\in S$ such that $\sum_{\tau=1}^t r_\star - r^J_\tau \le C_Jt^{\alpha_J}$
for all $t\le T$. Further, suppose we are given $N$ positive real numbers $\eta_1,\dots,\eta_N$. Set $R_i$ via:
\begin{align*}
    R_i&= C_iT^{\alpha_i}+\frac{(1-\alpha_i)^{\frac{1-\alpha_i}{\alpha_i}}(1+\alpha_i)^{\frac{1}{\alpha_i}}}{\alpha_i^{\frac{1-\alpha_i}{\alpha_i}}}C_i^{\frac{1}{\alpha_i}}T\eta_i^{\frac{1-\alpha_i}{\alpha_i}}\\
    &\qquad+  288\log(T^3 N/\delta) T \eta_i+\sum_{k\ne i} \frac{1}{ \eta_{k}}
\end{align*}
Then, with probability at least $1-3\delta$, the regret of Algorithm \ref{alg:combiner} satisfies:
\begin{align*}
    \regret  \le 3\sup_{j\in S} R_j&=\tilde O\left(\sup_{j\in S} C_jT^{\alpha_j}+C_j^{\frac{1}{\alpha_j}}T\eta_j^{\frac{1-\alpha_j}{\alpha_j}} \right.\\
    &\qquad\qquad\left.+ T\eta_j + \sum_{k\ne j} \frac{1}{\eta_k}\right)
\end{align*}
\end{restatable}
In most settings, $C_J\ge 1$ and $\eta \ge T^{-\alpha}$, so that $C_JT^{\alpha_J}$ is smaller than $C_J^{\frac{1}{\alpha_J}}T \eta_J^{\frac{1-\alpha_i}{\alpha_i}}$.

\begin{proof}[Proof sketch of Theorem \ref{thm:combiner}]
While the full proof of Theorem \ref{thm:combiner} is deferred to the appendix, we sketch the main ideas here. For simplicity, we consider the case that $\alpha_i=\frac{1}{2}$ for all $i$, assume that $S=\{J\}$ is a singleton set, assume $C_i\ge 1$, and drop all log factors and constants. Then by some martingale concentration bounds combined with the high-probability regret bound on $\A_J$, we have that $r_\star \le U(J,t-1) + \frac{R_J}{T}$ for all $t$ with high probability. Furthermore, by martingale concentration again, we have that $\hat \mu^i_t \le r_\star + \sqrt{t}$. Therefore, an algorithm is dropped from the set $I_t$ only if $\sum_{\tau=1}^{T(i,t)} r_\star - r^i_t \ge C_i\sqrt{T(i,t)}$, which does not happen for algorithm $\A_J$ with probability at least $1-\delta$. Further, by definition of $I_t$ and another martingale bound, all algorithms $i$ satisfy $\sum_{\tau=1}^{T(i,T)} \hat \mu^i_{\tau-1} - r^i_\tau\le C_i\sqrt{T(i,T)} +\sqrt{T(i,T)}\le 2C_i\sqrt{T(i,T})$ with high probability. Let us consider the instantaneous regret $r_\star - r_t$ on some round in which $i_t\ne J$. In this case, we must have $U(i_t,t-1) \ge U(J,t-1)$, so that we can write:
\begin{align*}
    r_\star - r_t &= r_\star - r^{i_t}_{T(i_t,t)}\\
    &= r_\star - U(J,t-1)+U(J,t-1)-U(i_t,t-1) \\
    &\qquad +U(i_t,t-1)- r^{i_t}_{T(i_t,t)}\\
    &\le \frac{R_J}{T} + \hat \mu^{i_t}_{T(i_t,t-1)} - r^{i_t}_{T(i_t,t)} - \frac{R_{i_t}}{T}\\
    &=\frac{R_J}{T} + \hat \mu^{i_t}_{T(i_t,t)-1} - r^{i_t}_{T(i_t,t)} - \frac{R_{i_t}}{T}
\end{align*}
Summing over all timesteps for which algorithm $i$ is chosen, we have
\begin{align*}
    \sum_{i_t =i}r_\star - r_t&\le \frac{R_JT(i,T)}{T}+2C_i\sqrt{T(i,T}) - \frac{R_iT(i,T)}{T}\\
    &\le \frac{R_JT(i,T)}{T}+\sup_{Z\ge 0} 2C_i\sqrt{Z} - \frac{R_iZ}{T}\\
    &\le \frac{R_JT(i,T)}{T} + \frac{C_i^2T}{R_i}
\end{align*}
Now summing over all indices $i\ne J$, we use the fact that $\sum_i T(i,T)\le T$ and the assumption that $\sum_{i\ne J} \frac{C_i^2T}{R_i}\le O(R_J)$ to conclude that the regret over all rounds in which $\A_J$ is not chosen is at most $O(R_J)$. For the rounds in which $A_J$ is chosen, we experience regret $C_J\sqrt{T(J,T)}\le C_J\sqrt{T}\le R_J$, which concludes the Theorem.
\end{proof}

%% file: gapbound.tex
\section{Gap-dependent regret bounds}\label{sec:gap}

In this section, we provide an analog of the standard ``gap-dependent'' bound for UCB. As a motivating example, consider the setting in which all $\A_i$ for $i\ne J$ never play any policy $x$ with $r(x)\ge r_\star - \Delta_i$ for some $\Delta_i\ge 0$. In this case, we might hope to perform much better, in the same way that standard UCB obtains logarithmic regret when the suboptimal arms have a non-negligible gap between their rewards and the optimal rewards. Specifically, we have the following result, whose proof is deferred to Section \ref{sec:appgapproof}:

\begin{restatable}{Theorem}{thmgap}\label{thm:gap}
Suppose that there is some $J\in\{1,\dots, N\}$ such that with probability at least $1-\delta$, we have:
\begin{align*}
    \sum_{\tau=1}^t r_\star - r_\tau^J \le C_J t^{\alpha_J}
\end{align*}
Also, for all $i$, define $T_i=T(i,T)$ and $\Delta_i$ by:
\begin{align*}
    \Delta_i =\frac{1}{T(i,T_i-1)} \sum_{\tau=1}^{T(i,T_i-1)} r_\star - r_\tau^i 
\end{align*}
And let $B\subset \{1,\dots,N\}$ with $J\notin B$ be the set of indices with $\Delta_i>\frac{2R_J}{T}$ for $i\in B$.

For $i\ne J$, let $C_i>0$ and $\alpha_i\le 1$ for $i\ne J$ be arbitrary. Then with probability at least $1-3\delta$, the regret of Algorithm \ref{alg:combiner} satisfies:
\begin{align*}
    &\sum_{t=1}^T r_\star - r_t \le   \sum_{\tau=1}^{T(J,T)}r_\star -r_\tau^J +\sum_{i\in B}1+ \frac{512\log(T^3N/\delta)}{\Delta_i} + \frac{4^{\frac{1}{1-\alpha_i}} C_i^{\frac{1}{1-\alpha_i}}}{\Delta_i^{\frac{\alpha_i}{1-\alpha_i}}}\\
    &\ +\min\left[\sum_{k\ne J, k\notin B}2 R_J,\ R_J+ \sum_{k\ne J, k\notin B} \max\left[\frac{(1-\alpha_k)(1+\alpha_k)^{\frac{1}{1-\alpha_k}} (2C_k)^{\frac{1}{1-\alpha_k}}T^{\frac{\alpha_k}{1-\alpha_k}} }{\alpha_k R_k^{\frac{\alpha_k}{1-\alpha_k}}},\ \frac{288\log(T^3N/\delta) T}{R_k}\right]\right]
\end{align*}
\end{restatable}
Note that we have made no conditions on $R_i$ in this expression. In particular, consider the case that each algorithm $\A_i$ considers only a subset of the possible policies, and that $\A_J$ is the only algorithm that is allowed to choose the optimal policy with reward $r_\star$. Then $\Delta_i$ is at least the gap between the reward of the best policy available to $\A_i$ and $r_\star$. Thus for large enough $T$, $B$ will be all indices except $J$, so that the overall regret provided in Theorem \ref{thm:gap} is $\sum_{\tau=1}^{T(J,T)} r_\star - r_\tau^J + \tilde O\left( \sum_{i\ne J} \frac{C_i^{\frac{1}{1-\alpha_i}}}{\Delta_i^{\frac{\alpha_i}{1-\alpha_i}}}\right)$.

As another example of this Theorem in action, let us consider the setting studied by \cite{arora2020corralling}. Specifically, each $\A_i$ has a putative regret bound of $\sum_{\tau=1}^t r_\star -r^i_\tau \le \sqrt{k_i\log(t) t}$ for all $t\le T$ for some $k_i$, and $\A_J$ in fact obtains its bound. However, for all $i\ne J$, $\A_i$ also suffers $\sum_{\tau=1}^t r_\star -r^i_t\ge \Delta_i t$ for all $t$ for some constant $\Delta_i$. For example, this might occur if each $\A_i$ is restricted to some subset of actions that does not include the best action. Now, recall that we made no restrictions of $R_i$ in Theorem \ref{thm:gap}, so we are free to set $R_i=0$ for all $i$. Then, we will have $B=\{1,\dots,J-1,J+1,N\}$ and obtain the following Corollary:
\begin{Corollary}\label{thm:log}
Suppose $k_1,\dots,k_N$ are such that for some $J$, $\A_J$ guarantees $\sum_{\tau=1}^t r_\star -r^J_\tau \le \sqrt{k_J\log(t) t}$ for all $t\le T$. Further, suppose that for all $i\ne J$,  $\sum_{\tau=1}^t r_\star -r^i_t\ge \Delta_i t$ for all $t$ for some constant $\Delta_i$. Then with $C_i=\sqrt{k_i\log(T)}$, $\alpha_i=\frac{1}{2}$ and $R_i=0$, with probability at least $1-3\delta$, Algorithm \ref{alg:combiner} guarantees regret:
\begin{align*}
    \sum_{t=1}^{T(J,T)} r_\star - r^J_t +\sum_{i\ne J}\frac{512 \log(T^3 N/\delta)}{\Delta_i} + \frac{16 k_i\log(T)}{\Delta_i}
\end{align*}
\end{Corollary}
Notably, the first term is the \emph{actual} regret of if $\A_J$ rather than the regret \emph{bound} $C_J\sqrt{T}$. Thus if $\A_J$ outperforms this bound and obtains logarithmic regret, our combiner algorithm will also obtain logarithmic regret, which is not obviously possible using techniques based on the Corral algorithm \cite{agarwal2017corralling}. Note that this result also appears to improve upon \cite{arora2020corralling} (Theorem 4.2) by removing a $\log(T)$ factor, but this is because we have assumed knowledge of the time horizon $T$ in order to set $C_i$.

%% file: doubling.tex
\section{Unknown and In-Expectation Bounds on Base Algorithms}\label{sec:doubling}

In this section, we show how to remedy two surface-level issues with Algorithm \ref{alg:combiner}. First, we require knowledge of the values $C_i$ and $\alpha_i$. Second, we require a high-probability regret bound for the well-specified base algorithm $\A_J$. Here, we show that a simple duplication and doubling-based technique suffices to address both issues. 

First, let us gain some intuition for how to convert an in-expectation bound into a high-probability bound suitable for use in Theorem \ref{thm:combiner}. Suppose we are given an algorithm that maintains expected regret $C_iT^{\alpha_i}$. We duplicate this algorithm $M=O(\log_2(1/\delta))$ times. Then by Markov inequality, each individual duplicate obtains regret at most $2C_iT^{\alpha_i}$ with probability at least $1/2$. Therefore, with probability at least $1-\frac{1}{2^M}=1-\delta$, at least one of the duplicates obtains regret at most $2C_iT^{\alpha_i}$. Then in the terminology of Theorem \ref{thm:combiner}, we let $S$ be the set of duplicate algorithms and so we satisfy the hypothesis of the Theorem. This argument is slightly flawed as-is because we need an \emph{anytime} regret bound for the base algorithms, but it turns out this is fixable by another use of Markov and union bound inequality. We then use a variant on the doubling trick to avoid requiring knowledge of $C$ and $\alpha$. The full construction is described below in Theorem \ref{thm:expectation}, with proof in Appendix \ref{sec:doublingproof}.

\begin{restatable}{Theorem}{thmexpectation}\label{thm:expectation}
Suppose that for some $J$, there is some unknown $\bar C_J$ and $\bar \alpha_J$ such that $\A_J$ ensures $\E[\sum_{\tau=1}^t r_\star - r^J_\tau]\le \bar C_Jt^{\bar\alpha_J}$ for all $t \leq T$. Further, suppose we are given positive numbers $\eta_1,\dots,\eta_N$. Let $\delta\in (0,1)$ be some user-specified failure probability. For $M=\lceil \log_2(1/\delta) \rceil$ and $K=\lceil \log_2(T)\rceil $ and $L=\lceil \frac{\log_2(T)}{2}\rceil $, we duplicate each $\A_i$ $MKL$ times, specifying each duplicate $\A_{i,x,y,z}$ by a multi-index $(i,x,y,z)\in[N]\times [M]\times [K]\times [L]$. To each duplicate we associate $C_{i,x,y,z} = 2^y$ and $\alpha_{i,x,y,z}=\min\left(1, 1/2 + \frac{z}{\log(T)}\right)$. Let $\eta_{i,x,y,z} = \eta_i$. Specify $R_{i,x,y,z}$ as a function of $\eta_{i,x,y,z}$ as described in Theorem \ref{thm:combinereta}. Then with probability at least $1-3\delta$, Algorithm \ref{alg:combiner} guarantees regret:

\begin{align*}
    &\regret \le O\left(\bar C_{J}T^{\bar \alpha_J}+\bar C_{J}^{\frac{1}{\bar \alpha_J}} T^{\bar \alpha_J} \eta_J^{\frac{1-\bar \alpha_J}{\bar \alpha_J}}\right.\\
    &\qquad\left.+\log\left(\tfrac{T^3N}{\delta}\right)T\eta_J+\sum_{i=1}^N \frac{\log(1/\delta)\log^2(T)}{\eta_i}\right)
\end{align*}
so that the expected regret is bounded by:
\begin{align*}
    &\E[\regret] \le O\left(\bar C_{J}T^{\bar \alpha_J}+\bar C_{J}^{\frac{1}{\bar \alpha_J}} T^{\bar \alpha_J} \eta_J^{\frac{1-\bar \alpha_J}{\bar \alpha_J}}\right.\\
    &\qquad\left.+\log\left(\tfrac{T^3N}{\delta}\right) T \eta_J+\sum_{i=1}^N \frac{\log(1/\delta)\log^2(T)}{\eta_i} + T\delta\right)
\end{align*}
\end{restatable}

%% file: examples.tex
\section{Examples and Optimality}\label{sec:examples}

In this section, we provide some illustrative examples of how our approach can be used. We will also highlight a few examples in which our construction matches lower bounds. The proofs are straightforward applications of Theorems \ref{thm:combiner} and \ref{thm:combinereta}, and are deferred to Appendix \ref{sec:examplesproofs}.

\subsection{$K$-Armed Bandits}

For our first example, suppose that the space $A$ is a finite set of $K$ arms and $X$ consists of the $K$ constant functions mapping all contexts to a single arm. This setup describes the classic $K$-armed bandit problem. Let $N=K$ and suppose each $\A_i$ is a naive algorithm that simply pulls arm $i$ on every round. We consider a $\A_i$ to be well-specified if the $i$th arm is in fact the optimal arm, in which case it is clear we may set $C_i=0$, $\alpha_i=\frac{1}{2}$ for all $i$. In the high-probability setting, we let $S$ be the singleton set containing only the unknown optimal index. In this case, the conditions on $R_i$ of Theorem \ref{thm:combiner} correspond almost exactly (up to constants and log factors) with the pareto frontier for regret bounds described in \cite{lattimore2015pareto}, showing that using our construction in this setting allows us to match this lower bound frontier.

\subsection{Misspecified Linear Bandit}

For our second example, suppose that the space $A$ is a finite set of $K$ arms, and that the context $c_t$ is a constant $c_t=c$ and provides a feature $c(a)\in \R^d$ for each arm $a\in A$. The space of policies is the set of $K$ constant functions again. In this case, it is possible  the reward $r(c,a)$ is a fixed linear function $\langle \beta, c(a)\rangle$ for some $\beta\in \R^d$, in which case the linUCB algorithm \cite{chu2011contextual} can obtain regret $\tilde O( \sqrt{d \log(K) t})$. On the other hand, in general the reward might be totally unrelated to the context, in which case one might wish to fall back on the UCB algorithm which obtains regret $\tilde O(\sqrt{Kt})$. By setting $\A_1$ to be linUCB and $\A_2$ to be ordinary UCB, we say that $S=\{1\}$ if the rewards are indeed linear, and $S=\{2\}$ otherwise. Further, we set $C_1=\sqrt{d\log(K)}$, $C_2=\sqrt{K}$ and $\alpha_1=\alpha_2=\frac{1}{2}$. Now let $P$ and $Q$ be any two numbers such that $P Q=KT$ and both $P$ and $Q$ are greater than $\sqrt{d\log(k) T}$. Then appropriate application of Corollary \ref{thm:combinereta}, yields regret $O(P)$ in the linear setting and regret $O(Q)$ in general. This again matches the frontier of regret bounds for this scenario described in  Theorem 24.4 of \cite{lattimore2018bandit} (see also Lemma 6.1 of \cite{pacchiano2020model}). Formally, we have the following Corollary:
\begin{restatable}{Corollary}{thmmisspecified}\label{thm:misspecified}
Suppose $c_t:A\to \R^d$ for all $t$. Suppose $K\ge d\log(K)$. Let $\A_1$ be an instance of linUCB and $\A_2$ be an instance of the ordinary UCB algorithm. Let $P$ and $Q$ be any two numbers such that $PQ=KT$ and both are greater than $\sqrt{d\log(K) T}$. We consider two cases, either the reward is a linear function of $c_t$, or it is not. Then $\A_1$ guarantees regret $\tilde O(\sqrt{d\log(K)t})$ with probability $1-\delta$ in the first case, while $\A_2$ guarantees regret $\tilde O(\sqrt{K t})$ in the second case. Set $\eta_1 = \frac{P}{d\log(K) T}=\frac{K}{Qd\log(K)}$ and $\eta_2 = \frac{1}{P}=\frac{Q}{KT}$. Then using the $R_i$ construction of Corollary \ref{thm:combinereta}, with probability at least $1-3\delta$, we guarantee regret $\tilde O(P)$ with linear rewards, and $\tilde O(Q)$ otherwise.
\end{restatable}

Note that we leverage our ability to use non-uniform $\eta$ values in this Corollary. It is not so obvious how to obtain this full frontier using the prior uniform bound (\ref{eqn:corral}), although it is of course conceivable that more detailed analysis of prior algorithms might allow for this same result.

\subsection{Linear Model Selection}

For our third example, we consider the case of model selection for linear bandits. In this setting, the context $c_t$ again specifies features $c_t(a)\in \R^d$ for each arm $a\in A$, and we are guaranteed that the reward is a linear function of the context. The question now is whether the full $d$-dimensions are actually necessary. Specifically, if there is some $d_\star$ such that the reward is in fact a linear function of the first $d_\star$ coordinates of the context only, then we would like our regret to depend on $d_\star$ rather than $d$. This setting has been studied before in the context of a finite set of actions in \cite{foster2019model, chatterji2019osom}. These prior works impose some additional technical conditions on the distribution of rewards and contexts provided by the environment. Under their conditions, \cite{chatterji2019osom} obtains regret $\tilde O(\sqrt{d_\star T})$ while under somewhat weaker conditions, \cite{foster2019model} obtains regret $\tilde O(\sqrt{d_\star T} + T^{3/4})$. In contrast, we require no extra conditions, and obtain regret $\tilde O(d_\star\sqrt{T})$. The construction is detailed in the following Corollary:

\begin{restatable}{Corollary}{thmmodelselection}\label{thm:modelselection}
Suppose $c_t\in \R^d$ for all $t$ and the reward is always a linear function of $c_t$. Suppose that the reward is in fact purely a linear function of the first $d_\star$ coordinates of $c_t$. Suppose the action set $A$ has finite cardinality $K$. Let $\A_i$ be an instance of linUCB of \cite{chu2011contextual} restricted to the first $2^i$ coordinates of the context. Set $C_i=\sqrt{2^i \log(K)}$, $\alpha_i=\frac{1}{2}$, and $\eta_i=\frac{1}{\sqrt{T}}$. Then using the instantiation of Algorithm \ref{alg:combiner} from Theorem \ref{thm:combinereta}, we obtain regret $\tilde O\left( d_\star \log(K)\sqrt{T}\right)$.
If instead the set $A$ is infinite, let $\A_i$ be an instance of the linUCB algorithm for infinite arms \cite{abbasi2011improved, Dani2008StochasticLO} restricted to the first $2^i$ coordinates. Set $C_i=2^i$, $\alpha_i=\frac{1}{2}$, and $\eta_i=\frac{1}{\sqrt{T}}$. Then we obtain regret $\tilde O\left( d_\star^2\sqrt{T}\right)$.
\end{restatable}

%% file: experiments.tex
\section{Experimental Validation}
\label{sec:experiments}
\begin{figure*}[tbp]
\centering
  \begin{subfigure}{.4\textwidth}
  \centering
  \includegraphics[width=\linewidth]{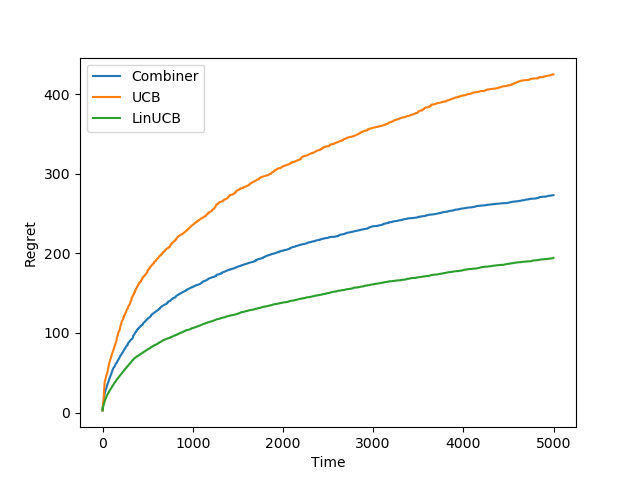}
  \caption{Linear Rewards}
  \label{fig:alpha0}
\end{subfigure}%
\begin{subfigure}{.4\textwidth}
  \centering
  \includegraphics[width=\linewidth]{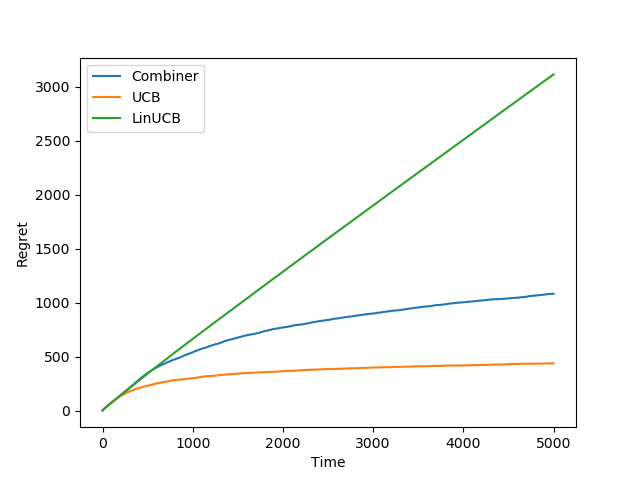}
  \caption{Non-Linear Rewards}
  \label{fig:alpha1}
\end{subfigure}
\caption{Misspecified Linear Bandit}
\label{fig:misspecified}
\end{figure*}
\begin{figure}[tbp]
    \centering
    \includegraphics[scale = 0.4]{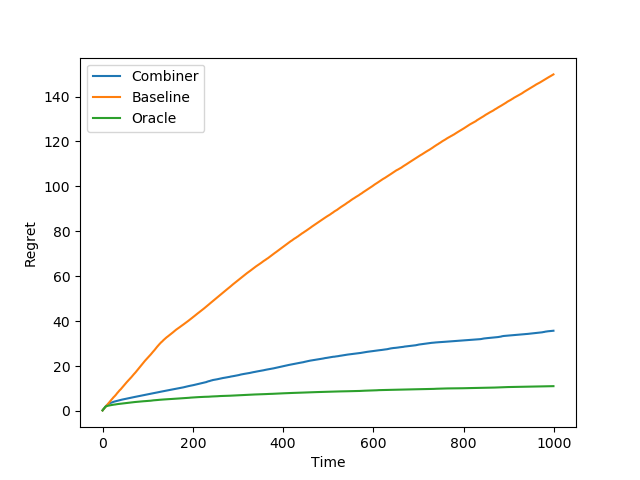}
    \caption{Model Selection Experiments}
    \label{fig:modelselection}
\end{figure}
We now demonstrate empirical validation of our results in two different application settings. For our first experiment, we consider the misspecified linear bandit setting. We use ordinary UCB and linUCB 
as the two base algorithms. For simplicity we focus on the ordinary stochastic bandits framework (i.e. we assume the context remains fixed over time). Each arm $a \in [K]$ is associated with a feature vector $x_a \in \mathbb{R}^d$ that is chosen from the uniform distribution on the unit sphere. Let $\beta \in \mathbb{R}^d$ also chosen from the unit sphere be a fixed unknown parameter vector. Finally, for each arm $a$, we choose $\mu_a\in \mathbb{R}$ to be specified later. The reward for arm $a$ at any time step $t$ is set to be $\alpha \mu_a + (1-\alpha) \sqrt{d} \cdot \langle \beta, x_a \rangle + \eta_t$ where $\eta_t \sim \mathcal{N}(0, \sigma)$ is independently sampled noise and $0 \leq \alpha \leq 1$ is a fixed constant. Let $a^\star = \argmin_{a \in [K]} \langle \beta, x_a \rangle$ be the worst arm with respect to the linear component of the reward. For any arm $a \neq a^\star$, we set $\mu_a = 0.25\sqrt{d} \cdot \langle \beta, x_a \rangle$, whereas we set $\mu_{a^\star} = 1$.
We consider two different settings of $\alpha$. When $\alpha = 0$, the reward is simply a linear function of the arms and we expect the linUCB algorithm to outperform the ordinary UCB algorithm. On the other hand, when $\alpha = 1$, the rewards are constructed so that the linUCB algorithm essentially never chooses the arm $a^\star$ and thus incurs linear regret whereas the ordinary UCB algorithm still guarantees $O(\sqrt{T})$ regret. Figures \ref{fig:alpha0} and \ref{fig:alpha1} show the performance on these two settings respectively. In both settings, the combiner uses UCB and linUCB as the base algorithms and the putative regret bounds for UCB and linUCB are computed empirically on independent instances of the non-linear and linear reward settings respectively.

For our second experiment, we consider the model selection problem in linear bandits. As earlier, each arm $a \in K$ is associated with a feature vector $x_a$ sampled independently from the unit sphere. The parameter vector $\beta \in \mathbb{R}^d$ is chosen so that first each $\beta_j \sim \mathcal{N}(0,1), \ \forall j \leq d^*$ and $\beta_j = 0$, otherwise, and then normalized to be of unit length. The reward for arm $a$ at any time step $t$ is set to be $\langle \beta, x_a \rangle + \eta_t$ where as earlier $\eta_t \sim \mathcal{N}(0,\sigma)$ is independently sampled noise. Figure \ref{fig:modelselection} demonstrates the performance of three different algorithms in this setting with $d=128, d^*=8$ and $K=1000$. The \emph{Baseline} algorithm is a vanilla LinUCB algorithm that works on the ambient dimension $d$, while the \emph{Oracle} is a LinUCB algorithm that works on the true dimensionality of the reward parameter $d^*$. The \emph{Combiner} algorithm is an implementation of Algorithm \ref{alg:combiner} with $\log d$ base algorithms where each base algorithm $A_i$ is an instance of LinUCB restricted to the first $2^i$ dimensions of the features. The putative regret bound $C_i$ for each base algorithm $A_i$ is computed empirically on independent instances where the corresponding algorithm is well-specified. We set the target regret bound $R_i \leftarrow (C_i^2 + N) \sqrt{T}$ where $N = \log d$ is the number of base algorithms. The experimental results validate our theoretical findings and show that the combiner algorithm is able to adapt to the true dimensionality of the rewards.


%% file: conclusion.tex
\section{Conclusion}\label{sec:conclusion}
We have introduced a new method for combining stochastic bandit algorithms in such a way that our final regret may depend only on the regret of the best base algorithm in hindsight. Our method is based on upper-confidence techniques, and provides a contrast to prior work based on mirror descent \cite{agarwal2017corralling}. We verify empirically that our technique can be used to solve some model selection and misspecification problems. In the future, we hope to see advancements in this area. For example, can we maintain logarithmic regret in benign settings? Further, in the full-information setting, one can often combine algorithms with minimal overhead. Are there any bandit settings in which such ideal behavior is possible?

%% file: mainalgproof.tex
\section{Lemmas for analysis of Algorithm \ref{alg:combiner}}\label{sec:mainalgproof_lemmas}
Before proving Theorem \ref{thm:combiner}, we need a few Lemmas. For the most part, these are straightforward verification of intuitive concentration bounds.

\begin{restatable}{Lemma}{thmmartingale}\label{thm:martingale}
With probability at least $1-\delta/T\ge 1-\delta$, for all $i$, for all $t$, we have:
\begin{align*}
    \sum_{\tau=1}^{T(i,t)} r^i_\tau - \hat r^i_\tau\le \sqrt{8\log(T^3N/\delta)T(i,t)}
\end{align*}
Similarly, with probability at least $1-\delta$, for all $i$, for all $t$, we have:
\begin{align*}
    \sum_{\tau=1}^{T(i,t)} \hat r^i_\tau -  r^i_\tau\le \sqrt{8\log(T^3N/\delta)T(i,t)}
\end{align*}
\end{restatable}
\begin{proof}
We prove only the first statement. The proof of the second is entirely symmetrical.

Fix a given $i$. For any $\tau\le T$ define $z_\tau =  r^i_\tau - \hat r^i_\tau$, we recall that $r^i_\tau$ is the expected value of the $\tau$th action of $A_i$ given the past history - if $t$ is the first time that $T(i,t)=\tau$, we have $r^i_\tau = \E[\hat r^i_\tau|\hat r_1,\dots,\hat r_{t-1}]$. Notice that $z_\tau\in[-1,1]$ is a martingale difference sequence.

Now for any $\epsilon$ we apply Azuma's inequality to obtain for any $t$:
\begin{align*}
    \text{Prob}\left[\sum_{\tau=1}^{t} r^i_\tau -\hat r^i_\tau\ge \epsilon\right] &=\text{Prob}\left[\sum_{\tau=1}^t z_\tau \ge \epsilon\right]\\
    &\le \exp\left(\frac{-\epsilon^2}{8t}\right)
\end{align*}
In particular, setting $\epsilon = \sqrt{8t\log(T^3N/\delta)}$, we have that
\begin{align*}
    \text{Prob}\left[\sum_{\tau=1}^{t}r^i_\tau -\hat r^i_\tau\ge \sqrt{8\log(T^3N/\delta)t}\right]\le \frac{\delta}{T^3N}
\end{align*}
Therefore:
\begin{align*}
\text{Prob}\left[\sum_{\tau=1}^{T(i,t)} r^i_\tau - \hat r^i_\tau\ge \sqrt{8\log(T^3N/\delta)T(i,t)}\right]&=\sum_{t'=1}^T\text{Prob}\left[\sum_{\tau=1}^{t'} r^i_\tau - \hat r^i_\tau\ge \sqrt{8\log(T^3N/\delta)t'}\text{ and } T(i,t)=t'\right]\\
&\le \sum_{t'=1}^T\text{Prob}\left[\sum_{\tau=1}^{t'} r^i_\tau - \hat r^i_\tau\ge \sqrt{8\log(T^3N/\delta)t'}\right]\\
&\le \frac{\delta}{NT^2}
\end{align*}
Then union bound over the $N$ indices $i$ and $T$ values of $t$ completes the Lemma.
\end{proof}

Similarly, we have the following result:

\begin{restatable}{Lemma}{thmempiricalregret}\label{thm:empiricalregret}
With probability at least $1-\delta$, for all $i$, for all $t$, we have:
\begin{align*}
    \sum_{\tau=1}^{T(i,t)}\hat \mu^i_{\tau} - r_\star \le 2\sqrt{8\log(T^3N/\delta)T(i,t)}
\end{align*}
\end{restatable}
\begin{proof}
Following the same argument as in Lemma \ref{thm:martingale}, we have that for any $i$ and $t$,
\begin{align*}
    \text{Prob}\left[\sum_{\tau=1}^{t}\hat r^i_\tau - r^i_\tau\ge \sqrt{8\log(T^3N/\delta)t}\right]\le \frac{\delta}{T^3N}
\end{align*}
Further, note that since $r^i_t\le r_\star$ for all $i$ and $t$, $\sum_{\tau=1}^{t}\hat r^i_\tau - r^i_\tau\le \sqrt{8\log(T^3N/\delta)t}$ implies $ \hat \mu^i_t -r_\star\le \frac{\sqrt{8\log(T^3N/\delta)}}{\sqrt{t}}$. Therefore, for any $i$ and $t$,
\begin{align*}
    \text{Prob}\left[\hat \mu^i_t -r_\star\ge \frac{\sqrt{8\log(T^3N/\delta)}}{\sqrt{t}}\right]\le \frac{\delta}{T^3N}
\end{align*}

Further, observe that if $\hat \mu^i_t -r_\star\le \frac{\sqrt{8\log(T^3N/\delta)}}{\sqrt{t}}$ for all $t$, then $\hat \sum_{\tau=1}^t\mu^i_\tau -r_\star\le 2\sqrt{8\log(T^3N/\delta)t}$. Therefore for any given $t$ and $i$:
\begin{align*}
    \text{Prob}\left[ \sum_{\tau=1}^{T(i,t)}\hat \mu^i_{\tau} - r_\star \ge 2\sqrt{8\log(T^3N/\delta)}\sqrt{T(i,t)}\right]&=\sum_{t'=1}^T \text{Prob}\left[ \sum_{\tau=1}^{t'}\hat \mu^i_{\tau} - r_\star \ge 2\sqrt{8\log(T^3N/\delta)}\sqrt{t'}\text{ and }T(i,t)=t'\right]\\
    &\le \sum_{t'=1}^T \text{Prob}\left[ \sum_{\tau=1}^{t'}\hat \mu^i_{\tau} - r_\star \ge 2\sqrt{8\log(T^3N/\delta)}\sqrt{t'}\right]\\
    &\le \sum_{t'=1}^T \text{Prob}\left[ \hat \mu^i_\tau -r_\star\le \frac{\sqrt{8\log(T^3N/\delta)}}{\sqrt{\tau}}\text{ for all }\tau\le t'\right]\\
    &\le \sum_{t'=1}^T \sum_{\tau=1}^{t'}\text{Prob}\left[ \hat \mu^i_\tau -r_\star\le \frac{\sqrt{8\log(T^3N/\delta)}}{\sqrt{\tau}}\right]\\
    &\le \frac{\delta}{TN}
\end{align*}
Now a union bound over the $T$ values of $t$ and $N$ values of $i$ completes the statement.
\end{proof}

\begin{Lemma}\label{thm:nodiscard}
With probability at least $1-3\delta$, there is some $J\in S$ such that for all $t$ we have simultaneously:
\begin{align*}
    \sum_{\tau=1}^{T(J,t)} r_\star - \hat r^J_\tau &\le C_JT(J,t)^{\alpha_J}\\
    r_\star - \hat \mu^J_{T(i,t-1)}&\le \min\left(1, C_JT(J,t-1)^{\alpha_J-1}+ \frac{\sqrt{8\log(T^3N/\delta)}}{\sqrt{T(J,t-1)}}\right)\\
    \sum_{\tau=1}^{T(J,t-1)} \hat \mu^J_{\tau-1} - \hat r^J_\tau& \le C_JT(J,t-1)^{\alpha_J} + 3\sqrt{8\log(TN/\delta)T(J,t-1)}
\end{align*}
\end{Lemma}
\begin{proof}
By assumption, with probability $1-\delta$ we have that there is some $J\in S$ such that for all $t$,
\begin{align}
    \sum_{t=1}^{T(J,t-1)} r_\star  - r^J_\tau \le C_JT(J,t-1)^{\alpha_J}\label{eqn:nodiscardregret}
\end{align}
Next, by Lemma \ref{thm:martingale} we have with probability at least $1-\delta$, for all $t$,
\begin{align}
    \sum_{\tau=1}^{T(J,t-1)}r^J_\tau - \hat r^J_\tau \le \sqrt{8\log(T^3N/\delta)T(J,t-1)}\label{eqn:nodiscardmartingale}
\end{align}
and by Lemma \ref{thm:empiricalregret}, with probability $1-\delta$ we have for all $t$:
\begin{align}
    \sum_{\tau=1}^{T(J,t)}\hat \mu^J_{\tau} - r_\star \le 2\sqrt{8\log(T^3N/\delta)T(J,t)}\label{eqn:nodiscardempirical}
\end{align}
All three equations then hold with probability at least $1-3\delta$. Let us condition on this probability $1-3\delta$ event. Then the first equation in the Lemma is immediately satisfied. 

For the second equation, observe that since $\hat \mu^J_{T(J,t-1)}$ and $r_\star$ are both in $[0,1]$, the statement is true so long as the minimum binds to 1. Further, for $T(J,t-1)\ne 0$, we have
\begin{align*}
    r_\star - \hat \mu^J_{T(J,t-1)}&=\frac{1}{T(J,t-1)}\sum_{\tau=1}^{T(J,t-1)} r_\star - \hat r^J_\tau \\
    &=\frac{1}{T(J,t-1)}\sum_{\tau=1}^{T(J,t-1)} r_\star -  r^J_\tau+\frac{1}{T(J,t-1)}\sum_{\tau=1}^{T(J,t-1)} r^J_\tau - \hat r^J_\tau\\
    &\le C_JT(J,t-1)^{\alpha_J-1}+ \frac{\sqrt{8\log(T^3N/\delta)}}{\sqrt{T(J,t-1)}}
\end{align*}

Now for the third equation, observe $\hat \mu^J_0=0$, so that
\begin{align*}
    \sum_{\tau=1}^{T(J,t-1)}\hat \mu^i_{\tau-1} - r_\star &\le\sum_{\tau=1}^{T(J,t-1)}\hat \mu^i_{\tau} - r_\star\\
    &\le \sum_{\tau=1}^{T(J,t-1)}\frac{\sqrt{8\log(T^3N/\delta)}}{\sqrt{\tau}}\\
    &\le 2\sqrt{8\log(T^3N/\delta)T(J,t-1)}
\end{align*}
Combining these equations yields:
\begin{align*}
    \sum_{\tau=1}^{T(J,t-1)} \hat \mu^J_{\tau-1} - \hat r^J_\tau &= \sum_{\tau=1}^{T(J,t-1)} \hat \mu^J_{\tau-1} - r_\star + r_\star - r^J_\tau+ r^J_\tau - \hat r^J_\tau\\
    &\le 2\sqrt{8\log(T^3N/\delta)T(J,t-1)} + C_JT(J,t-1)^{\alpha_J} + \sqrt{8\log(T^3N/\delta)T(J,t-1)}\\
    &\le C_JT(J,t-1)^{\alpha_J} + 3\sqrt{8\log(T^3N/\delta)T(J,t-1)}
\end{align*}
\end{proof}

We also need a bound on the performance of the misspecified algorithms:
\begin{Lemma}\label{thm:limitbadalgs}
With probability at least $1-\delta$, all indices $i$ satisfy:
\begin{align*}
    \sum_{\tau=1}^{T(i,T)} \hat \mu^i_{\tau-1} -  r^i_\tau\le C_iT(i,T)^{\alpha_i} + 4\sqrt{8\log(TN/\delta)T(i,T)}+1
\end{align*}
\end{Lemma}
\begin{proof}
Let $t$ be the smallest time such that $T(i,t) = T(i,T)$. Then $i_t = i$, so we must have that $i\in I_{t}$. Therefore:
\begin{align*}
\sum_{\tau=1}^{T(i,t-1)} \hat \mu^i_{\tau-1} - \hat r^i_\tau\le C_iT(i,t-1)^{\alpha_i} + 3\sqrt{8\log(T^3N/\delta)(T(i,t-1))}
\end{align*}
Next, by Lemma \ref{thm:martingale}, we have that with probability at least $1-\delta$,
\begin{align*}
    \sum_{\tau=1}^{T(i,t-1)} \hat r^i_\tau - r^i_\tau &\le \sqrt{8\log(T^3N/\delta)T(i,t-1)}
\end{align*}
Then observe that since all rewards are in $[0,1]$, we must have $\hat \mu^i_{T(i,t)-1} - r^i_{T(i,t)}\le 1$ and the statement follows.
\end{proof}

One more technical fact:
\begin{restatable}{Lemma}{thmalphabound}\label{thm:alphabound}
Suppose $\alpha \in [1/2,1)$. Then for any constants $A,B\ge 0$,
\begin{align*}
    \sup_{Z\ge 0} A Z^\alpha - BZ&=\alpha^{\frac{\alpha}{1-\alpha}}\left(1 -\alpha\right) \frac{A^{\frac{1}{1-\alpha}}}{B^{\frac{\alpha}{1-\alpha}}}
\end{align*}
\end{restatable}
\begin{proof}
We differentiate with respect to $Z$ to obtain:
\begin{align*}
    \alpha A Z^{\alpha-1} &= B\\
    Z&=\left(\frac{\alpha A}{B}\right)^{\frac{1}{1-\alpha}}\\
    \intertext{Now substitute this value back into the original expression:}
    AX^\alpha - BX &=\left(\alpha^{\frac{\alpha}{1-\alpha}} -\alpha^{\frac{1}{1-\alpha}}\right) \frac{A^{\frac{1}{1-\alpha}}}{B^{\frac{\alpha}{1-\alpha}}}\\
    &=\alpha^{\frac{\alpha}{1-\alpha}}\left(1 -\alpha\right) \frac{A^{\frac{1}{1-\alpha}}}{B^{\frac{\alpha}{1-\alpha}}}
\end{align*}

\end{proof}

\section{Proof of Theorem \ref{thm:combiner}}\label{sec:mainalgproof}
Now we are in a position to prove Theorem \ref{thm:combiner}, which we restate below for reference:
\thmcombiner*
\begin{proof}
To begin, observe that by Lemma \ref{thm:nodiscard} and Lemma \ref{thm:limitbadalgs}, with probability $1-4\delta$, there is some $J\in S$ such that for all $t$ and all $i$, we have
\begin{align}
    \sum_{\tau=1}^{T(J,T)}r_\star - r^J_\tau &\le C_JT(J,T)^{\alpha_J}\label{eqn:goodregret}\\
    r_\star - \hat\mu^J_{T(J,t-1)}&\le \min\left(1, C_JT(J,t-1)^{\alpha_J-1}+\frac{\sqrt{8\log(T^3N/\delta)}}{\sqrt{T(J,t-1)}}\right)\label{eqn:bound}\\
    J&\in I_t\nonumber\\
    \sum_{\tau=1}^{T(i,T)} \hat \mu^i_{\tau-1} -  r^i_\tau&\le C_iT(i,T)^{\alpha_i} + 4\sqrt{8\log(TN/\delta)T(i,T)}+1\label{eqn:limitbad}
\end{align}
All of our analysis is conditioned on this probability $1-4\delta$ event.

First, notice that by (\ref{eqn:bound}), we have that we have that:
\begin{align*}
r_\star &\le  U(J,t-1) + \frac{R_J}{T}
\end{align*}
for all $t$.

Now we write the regret:
\begin{align*}
    \sum_{t=1}^T r_\star -r_t&=\sum_{t=1}^T r_\star -  r^{i_t}_{T(i_t,t)}\\
    &=\sum_{i_t=J}r_\star - r^J_{T(J,t)} + \sum_{i_t\ne J} r_\star - r^{i_t}_{T(i_t,t)}
\end{align*}
First let's consider the indices for which $i_t=J$:
\begin{align*}
    \sum_{i_t=J}r_\star - r^J_{T(J,t)} \le C_JT(J,T)^{\alpha_J}\le R_J
\end{align*}

Next, if $i_t\ne J$, we must have $U(i_t, t-1)\ge U(J, t-1)$ since $J\in I_t$ for all $t$. Then:
\begin{align*}
    r_\star &= r_\star - U(J,t-1) + U(J,t-1) - U(i_t,t-1) + U(i_t,t-1)\\
    &\le \frac{R_J}{T}+U(i_t, t-1)
\end{align*}
Therefore:
\begin{align*}
    \sum_{i_t\ne J} r_\star - r^{i_t}_{T(i_t,t)}&\le \sum_{i_t\ne J} \frac{R_J}{T} + U(i_t,t-1) - r^{i_t}_{T(i_t,t)}\\
    &\le R_J  + \sum_{i_t\ne J} U(i_t,t-1) - r^{i_t}_{T(i_t,t)}
\end{align*}

Let us focus on the sum $\sum_{i_t\ne J} U(i_t,t-1) -  r^{i_t}_{T(i_t,t)}$. As a first step, we observe the following identity:
\begin{align*}
    \sum_{t\ |\ i_t =i}U(i, t-1)-\hat \mu^i_{T(i,t-1)}&=\sum_{\tau=1}^{T(i,T)}\min\left(1,  C_i(\tau-1)^{\alpha_i -1} +\frac{\sqrt{8\log(T^3N/\delta)}}{\sqrt{\tau-1}}\right)-\frac{R_i}{T}\\
    &\le 1-\frac{T(i,T)R_i}{T}+ \sum_{\tau=1}^{T(i,T)}C_i\tau^{\alpha_i -1} +\frac{\sqrt{8\log(T^3N/\delta)}}{\sqrt{\tau}}\\
    &\le 1 -\frac{T(i,T)R_i}{T}+ \frac{C_i T(i,T)^{\alpha_i}}{\alpha_i} + 2\sqrt{8\log(T^3N/\delta)T(i,T)}\\
\end{align*}


Now we use this observation as follows:
\begin{align*}
    \sum_{i_t\ne J} U(i_t,t-1) - \hat r^{i_t}_{T(i_t,t)}&=\sum_{i_t\ne J} \hat \mu^{i_t}_{T(i_t,t-1)} -  r^{i_t}_{T(i_t,t)} + U(i_t,t-1) - \hat\mu^{i_t}_{T(i_t,t-1)}\\
    &\le \sum_{i\ne J}\left[\sum_{t\ |\ i_t =i}\hat \mu^{i}_{T(i,t-1)} -  r^{i_t}_{T(i_t,t)}+U(i, t-1)-\hat \mu^i_{T(i,t-1)}\right] \\
    &\le \sum_{i\ne J} 1 -\frac{T(i,T)R_i}{T}+ \frac{C_i T(i,T)^{\alpha_i}}{\alpha_i} + 2\sqrt{8\log(T^3N/\delta)T(i,T)} + \sum_{\tau=1}^{T(i,T)} \hat \mu^i_{\tau-1}-r^i_\tau\\
    &\le \sum_{i\ne J}  -\frac{T(i,T)R_i}{T}+ \frac{C_i T(i,T)^{\alpha_i}}{\alpha_i} + 3\sqrt{8\log(T^3N/\delta)T(i,T)} + \sum_{\tau=1}^{T(i,T)} \hat \mu^i_{\tau-1}-r^i_\tau\\
    &\le  \sum_{i\ne J}  \frac{(1+\alpha_i)C_i T(i,T)^{\alpha_i}}{\alpha_i} + 6\sqrt{8\log(T^3N/\delta)T(i,T)} - \frac{T(i,T)R_i}{T}
\end{align*}
Now we use a simple trick that allows us to apply Lemma \ref{thm:alphabound}:
\begin{align*}
    &\frac{(1+\alpha_i)C_i T(i,T)^{\alpha_i}}{\alpha_i} + 6\sqrt{8\log(T^3N/\delta)T(i,T)} - \frac{T(i,T)R_i}{T}\\
    &\le \sup_Z \frac{(1+\alpha_i)C_i Z^{\alpha_i}}{\alpha_i} + 6\sqrt{8\log(T^3N/\delta)Z} - \frac{ZR_i}{T}\\
    &\le  \max\left[\sup_Z \frac{(1+\alpha_i)2C_i Z^{\alpha_i}}{\alpha_i} - \frac{ZR_i}{T}, \sup_{Z} 12\sqrt{8\log(T^3N/\delta)Z}- \frac{ZR_i}{T}\right]\\
    &\le  \max\left[\frac{(1-\alpha_i)(1+\alpha_i)^{\frac{1}{1-\alpha_i}} (2C_i)^{\frac{1}{1-\alpha_i}}T^{\frac{\alpha_i}{1-\alpha_i}} }{\alpha_i R^{\frac{\alpha_i}{1-\alpha_i}}}, \frac{288\log(T^3N/\delta) T}{R_i}\right]
\end{align*}
Therefore by our conditions on the $R_i$, we have
\begin{align*}
    &\sum_{i\ne J} \frac{(1+\alpha_i)C_i T(i,T)^{\alpha_i}}{\alpha_i} +6\sqrt{8\log(T^3N/\delta)T(i,T)} - \frac{T(i,T)R_i}{T}\\
    &\le \sum_{i\ne J} \max\left[\frac{(1-\alpha_i)(1+\alpha_i)^{\frac{1}{1-\alpha_i}} (2C_i)^{\frac{1}{1-\alpha_i}}T^{\frac{\alpha_i}{1-\alpha_i}} }{\alpha_i R^{\frac{\alpha_i}{1-\alpha_i}}}, \frac{288\log(T^3N/\delta) T}{R_i}\right]\\
    &\le R_J
\end{align*}
And so putting all this together we have with probability at least $1-3\delta$:
\begin{align*}
    \sum_{t=1}^T r_\star -\hat r_t&=\sum_{i_t=J}r_\star - \hat r^J_{T(J,t)} + \sum_{i_t\ne J} r_\star - \hat r^{i_t}_{T(i_t,t)}\\
    &\le R_J + R_J+R_J\le 3R_J
\end{align*}
for \emph{some} $J\in S$. Therefore with probability at least $1-3\delta$, the regret is bounded by $3\sup_{j\in S} R_j$.
\end{proof}

\subsection{Proof of Corollary \ref{thm:combinereta}}\label{sec:combineretaproof}

Now we prove Corollary \ref{thm:combinereta}:
\thmcombinereta*

\begin{proof}

Observe that by definition of $R_i$ we have:
\begin{align*}
    R_i&\ge C_iT^{\alpha_i}\\
    \intertext{Further, due to the second and third terms in the definition of $R_i$ respectively, we have the following:}
    \frac{1}{\eta_i} &\geq 
    \frac{(1-\alpha_i)(1+\alpha_i)^{\frac{1}{1-\alpha_i}}(2C_i)^{\frac{1}{1-\alpha_i}}T^{\frac{\alpha_i}{1-\alpha_i}}}{\alpha_i R_i^{\frac{\alpha_i}{1-\alpha_i}}} \\
    \frac{1}{\eta_i} &\geq
    \frac{288 \log(T^3 N/\delta) T}{R_i}\\
    \intertext{Finally, due to the last term in the definition of $R_i$, we have:}
    R_i&\ge \sum_{k\ne i} \frac{1}{\eta_k} \\
    &\ge \sum_{k \neq i} \max\left\{\frac{(1-\alpha_k)(1+\alpha_k)^{\frac{1}{1-\alpha_k}}(2C_k)^{\frac{1}{1-\alpha_k}}T^{\frac{\alpha_k}{1-\alpha_k}}}{\alpha_k R_k^{\frac{\alpha_k}{1-\alpha_k}}}, 
    \frac{288 \log(T^3 N/\delta) T}{R_k} \right\}
\end{align*}
Therefore these $R_i$ satisfy the conditions of Theorem \ref{thm:combiner}, and so we are done.
\end{proof}

%% file: gapboundproof.tex
\section{Proof of Theorem \ref{thm:gap}}\label{sec:appgapproof}

\thmgap*

\begin{proof}
We begin with the regret decomposition:
\begin{align*}
    \sum_{t=1}^T r_\star -r_t = \sum_{\tau=1}^{T(J,T)} r_\star - r^J_\tau + \sum_{i\notin B, i\ne J}\sum_{\tau=1}^{T(i,T)} r_\star - r^i_\tau + \sum_{i\in B}\sum_{\tau=1}^{T(i,T)}r_\star - r^i_\tau
\end{align*}
The first sum in the above is already in the form we desire. The second sum could be bounded using the same argument as in the proof of Theorem \ref{thm:combiner} by:
\begin{align*}
    R_J+\sum_{k\ne J, k\notin B} \max\left[\frac{(1-\alpha_k)(1+\alpha_k)^{\frac{1}{1-\alpha_k}} (2C_k)^{\frac{1}{1-\alpha_k}}T^{\frac{\alpha_k}{1-\alpha_k}} }{\alpha_k R_k^{\frac{\alpha_k}{1-\alpha_k}}},\ \frac{288\log(T^3N/\delta) T}{R_k}\right]
\end{align*}
However, we can also bound it another way. By definition of $B$, for any $i\notin B$, we have:
\begin{align*}
    \sum_{\tau=1}^{T(i,T)} r_\star - r^i_\tau \le T\Delta_i \le 2R_J
\end{align*}
so that the second sum is also bounded by
\begin{align*}
    \sum_{k\ne J, k\notin B} 2R_J
\end{align*}
This leads to the expression involving a minimum in the Theorem statement.

For the last term, our proof proceeds by counting how many times we choose an algorithm $\A_i$ for $i\in B$. We only choose such an algorithm when $U(i,t-1)\ge U(J,t-1)$. Further, note from the proof of Theorem \ref{thm:combiner}, that with probability at least $1-\delta$, we have:
\begin{align*}
    r_\star \le U(J,t-1) + \frac{R_J}{T}
\end{align*}
for all $t$. For any $i$, define $T_i$ as the last time at which $\A_i$ is chosen by Algorithm \ref{alg:combiner}. Our goal is to bound $T_i$ for all $i\in B$. By definition, we must have $U(i,T_i-1)\ge U(J,T_i-1)$. Then, observe that from Lemma \ref{thm:martingale}, we have that with probability at least $1-\delta$, for all $i\in B$:
\begin{align*}
     \hat \mu_{T(i,T_i-1)}^i - r_\star +\Delta_i =\frac{1}{T(i,T_i-1)}\sum_{\tau=1}^{T(i,T_i-1)} \hat r_\tau^i - r_\tau^i\le \frac{2\sqrt{8\log(T^3N/\delta)}}{\sqrt{T(i,T_i-1)}}
\end{align*}
Define $Z(i, t) = \min\left(1, \frac{C_iT(i,t)^\alpha_i + \sqrt{8\log(T^3N/\delta)T(i,t)}}{T(i,t)}\right)$. Then by the previous line and by definition of $U(i,T_i-1)$, we have with probability at least $1-2\delta$,
\begin{align*}
    U(i, T_i-1) &\le \frac{2\sqrt{8\log(T^3N/\delta})}{\sqrt{T(i,T_i-1)}} + r_\star -\Delta_i+ Z(i, T_i-1)-\frac{R_i}{T}\\
    &\le \frac{2\sqrt{8\log(T^3N/\delta})}{\sqrt{T(i,T_i-1)}} + U(J,T_i-1) + \frac{R_J}{T} -\Delta_i+Z(i,T_i-1)-\frac{R_i}{T}
\end{align*}
Therefore, if $U(i,T_i-1)\ge U(J,T_i-1)$, we must have:
\begin{align}
    \Delta_i +\frac{R_i}{T} - \frac{R_J}{T}&\le  \frac{2\sqrt{8\log(T^3N/\delta})}{\sqrt{T(i,T_i-1)}}+\min\left(1, \frac{C_iT(i,T_i-1)^\alpha_i + \sqrt{8\log(T^3N/\delta)T(i,T_i-1)}}{T(i,T_i-1)}\right)\label{eqn:delta}\\
    \frac{\Delta_i}{2}&\le \frac{2\sqrt{8\log(T^3N/\delta})}{\sqrt{T(i,T_i-1)}}+\min\left(1, \frac{C_iT(i,T_i-1)^\alpha_i + \sqrt{8\log(T^3N/\delta)T(i,T_i-1)}}{T(i,T_i-1)}\right)\nonumber\\
    T(i,T_i-1)&\le \frac{512\log(T^3N/\delta)}{\Delta_i^2} + \frac{4^{\frac{1}{1-\alpha_i}} C_i^{\frac{1}{1-\alpha_i}}}{\Delta_i^{\frac{1}{1-\alpha_i}}}\nonumber
\end{align}

Therefore, observing that by definition we have $T(i,T_i-1)+1=T_i = T(i,T)$, we have:
\begin{align*}
    \sum_{\tau=1}^{T(i,T)} r_\star - r^i_\tau &\le 1+ T(i,T_i-1)\Delta_i\\
    &\le1+\frac{512\log(T^3N/\delta)}{\Delta_i} + \frac{4^{\frac{1}{1-\alpha_i}} C_i^{\frac{1}{1-\alpha}}}{\Delta_i^{\frac{\alpha_i }{1-\alpha_i}}}
\end{align*}


\end{proof}

%% file: doublingproof.tex
\section{Proof of Theorem \ref{thm:expectation}}\label{sec:doublingproof}
In this section, we provide a proof of Theorem \ref{thm:expectation}, restated below:
\thmexpectation*

\begin{proof}
Let $\bar y=\lceil \log_2(\bar C_J)\rceil$, and let $\bar z =\lceil \frac{\bar \alpha_J - \frac{1}{2}}{\log_2(T)}$. Then observe since $T^{1/\log_2(T)}=2$, we have that for all $x$ and $t\le T$, 
\begin{align*}
\bar C_{J} t^{\bar \alpha_J}\le C_{J,x,\bar y,\bar z} t^{\alpha_{J,x,\bar y,\bar z}} \le 4 \bar C_{J} t^{\bar \alpha_J}
\end{align*}

Next, let $\bar S=\{(J,x,\bar y,\bar z): x\in[M]\}$. Then we claim that with probability at least $1-\delta$, there exists some $\bar J= (J, \bar x,\bar u,\bar z)\in \bar S$ such that such that $\A_{\bar J}$ obtains regret $\sum_{\tau=1}^t r_\star - r^{\bar J}_\tau\le C_{J,x,\bar y,\bar z} t^{\alpha_{J,x,\bar y,\bar z}}$ for all $t\le T$. The conclusion of the Theorem will then follow from Theorem \ref{thm:combinereta}, observing that the total number of duplicates for each algorithm is $MKL=O(\log(1/\delta)\log^2(T))$.

Let us prove the claim. Define $V=\lceil \log_2(T) \rceil$ and define the sequence $t_1,\dots,t_V$ by $t_i=2^i$ for $i<V$ and $t_V=T$. Then for any $\epsilon>0$ and any $i$ and any $j$, by Markov inequality, with probability at least $1-\epsilon$ we have
\begin{align*}
   \sum_{\tau=1}^{t_i} r_\star - r^{J,j,\bar y,\bar z}_\tau \le \frac{C_{J,j,\bar y, \bar Z}t_i^{\alpha_{J,j,\bar y, \bar z}}}{\epsilon}  
\end{align*}
With $\epsilon = \frac{1}{2(\log_2(T) +1)}$, by union bound over the $V\le \log_2(T) +1$ values of $i$, we have that for any given $j$, with probability at least $1/2$,
\begin{align}
   \sum_{\tau=1}^{t_i} r_\star - r^{J,j,\bar y,\bar z}_\tau \le 2(\log_2(T)+1)C_{J,j,\bar y, \bar Z}t_i^{\alpha_{J,j,\bar y, \bar z}}\label{eqn:jbound}
\end{align}
for all $i$.

There are $M$ multi-index $\bar J$ of the form $(J,j,\bar y,\bar z)$. Thus the probability that \emph{none} of them satisfies (\ref{eqn:jbound}) for all $i$ is at most $\frac{1}{2^M}\le \delta$. Therefore, with probability at least $1-\delta$, we have that (\ref{eqn:jbound}) holds for all $i$ \emph{some} $j$ and multi-index $\bar J=(J,j,\bar y,\bar z)$. 

Next, observe that $r_\star - r^{\bar J}_t\ge 0$ for all $t$ with probability 1. Therefore, with probability $1-\delta$, for any $t$ we set $t_i$ such that $t\le t_i\le 2t$ and have
\begin{align*}
    \sum_{\tau=1}^t r_\star - r^{\bar J}_\tau &\le \sum_{\tau=1}^{t_i} r_\star - r^{\bar J}_\tau\\
    &\le 2C(\log_2(T)+1)\sqrt{t_i}\\
    &\le 4C(\log_2(T)+1)\sqrt{t}
\end{align*}

\end{proof}

%% file: examplesproofs.tex
\section{Proofs for Section \ref{sec:examples}}\label{sec:examplesproofs}

\thmmisspecified*
\begin{proof}
The regret bounds on $\A_1$ and $\A_2$ are standard from analysis of the respective algorithms. We set $\alpha_1=\alpha_2=\frac{1}{2}$ and set $C_1=\tilde O(\sqrt{d \log(K)})$ and $C_2 = \tilde O(\sqrt{K})$, where the constants and log factors in the $\tilde O$ notations come from the regret bounds on $\A_1$ and $\A_2$. Then if the rewards are linear, then applying Theorem \ref{thm:combinereta} yields regret $\tilde O\left(d\log(K) T\eta_1 + \frac{1}{\eta_2}\right) = \tilde O(X)$ (with $S=\{1\}$), while otherwise we have $\tilde O\left(KT \eta_2 + \frac{1}{\eta_1}\right)=\tilde O\left(Y + \frac{d\log(K) Y}{K}\right)=\tilde O\left(Y\right)$ (with $S=\{2\}$).
\end{proof}

\thmmodelselection*
\begin{proof}
For the finite action setting, \cite{chu2011contextual} proves, supposing $d_\star \le \hat d$, then an instance of linUCB restricted to the first $\hat d$ dimensions will obtain regret at most $\tilde O(\sqrt{\hat d \log(K) T})$. In the infinite action setting, \cite{abbasi2011improved} shows that the OFUL algorithm obtains regret $\tilde O(\hat d \sqrt{T})$. The corollary is then a consequence of Theorem \ref{thm:combinereta}.
\end{proof}

%% file: ucbbounds.tex
\section{High-Probability Regret and the Optimism Principle}
\label{sec:ucb}
In this section we sketch a proof that UCB-like algorithms tend to satisfy the high-probability conditions required by Theorem \ref{thm:combiner}, without resorting to the duplication technique needed in Theorem \ref{thm:expectation}. The analysis in this section is standard and more detailed discussion can be found in various texts (e.g. \cite{lattimore2018bandit, BubeckC12, Slivkins19}).

UCB and related algorithms make use of the \emph{optimism principle}. The intuitive idea is to maintain for each policy $x$ a bound on the best possible value for the expected reward of $x$ given the data seen so far. At each round, the policy with the largest such bound is chosen. In general, the bound for a policy is some mixture of its past observed performance and some uncertainty due to lack of data, so that the upper bound is larger both if a policy has done well in the past, or if it has not been explored in the past. More formally, we consider an algorithm UCB-like if for each policy $x$ and round index $t$ it produces a confidence set $C(x,t,\delta)\subset \R$ such that with probability at least $1-\delta$, the expected reward $r_x$ satisfies $r_x\in C(x,t,\delta)$ for all $x$ for all $t$. At each time step, the algorithm plays the arm $x_t=\argmax_{x} U(x,t,\delta)$, where $U(x,t,\delta) = \sup C(x,t,\delta)$. Conditioned on this probability $1-\delta$ event, the analysis proceeds as follows:
\begin{align*}
    r_\star - r_{x_t} &\le r_\star - U(\star, t, \delta) +U(\star,t,\delta) - U(x_t, t, \delta)+ U(x_t, t, \delta)-r_{x_t}\\
    &\le U(x_t, t, \delta) - r_{x_t}\\
    &\le U(x_t, t, \delta)-L(x_t,t,\delta)
\end{align*}
where $L(x, t,\delta)=\inf C(x,t,\delta)$.
Thus, with probability at least $1-\delta$, the total regret at any time $\tau\le T$ is
\begin{align*}
    \sum_{t=1}^\tau r_\star - r_t\le \sum_{t=1}^\tau U(x_t, t, \delta)-L(x_t,t,\delta)
\end{align*}
Many UCB-like bounds operate by providing a bound on $\sum_{t=1}^T U(x_t, t, \delta)-L(x_t,t,\delta)$ that holds for \emph{every possible sequence of $x_t$} - not just those ones that are chosen by the algorithm. Algorithms with regret bounds that can be derived in this manner include the standard UCB algorithm for $K$-armed bandits \cite{lai1985asymptotically}, as well as the linear UCB algorithm for potentially infinitely armed linear bandits \cite{Dani2008StochasticLO,abbasi2011improved}. In the case of ordinary UCB, $C(x,t,\delta)=[\mu_x - \frac{C(\delta)}{\sqrt{T(x,t-1)}}, \mu_x + \frac{C(\delta)}{\sqrt{T(x,t-1)}}]$ where $C$ is some constant that is $O(\sqrt{\log(T/\delta)})$ and $\mu_x$ is the empirical average of all prior rewards observed from arm $x$. Here we observe that $U(x_t, t,\delta)-L(x_t, t,\delta)$, conditioned on the sequence of arm pulls, is independent of the actual observed rewards. It can be shown that
\begin{align*}
    \sum_{t=1}^\tau U(x_t, t, \delta)-L(x_t,t,\delta)\le O(C(\delta)\sqrt{ K \tau})
\end{align*}
for any sequence $x_1,\dots,x_\tau$. Therefore so long as the confidence intervals are correct for all times $t$, which happens with probability at least $1-\delta$, we have a bound on the regret of $O(C(\delta)\sqrt{ K \tau})$

In the case of linUCB, we have the more complicated set:
\begin{align*}
    C(x,t,\delta) = [\langle \hat \theta , x\rangle - C(\delta)\sqrt{x^\top V_{t-1}x}, \langle \hat \theta , x\rangle + C(\delta)\sqrt{x^\top V_{t-1}x}]
\end{align*}
where $\hat\theta$ is the regularized least-squares regression solution for the linear parameter $\theta$, $V_{t-1}=\lambda I + \sum_{\tau=1}^{t-1} x_{i_\tau} x_{i_\tau}^\top$, and $C(\delta)$ again is constant that is $O(\sqrt{d\log(T/\delta)})$. Again, we observe that $U(x_t, t,\delta)-L(x_t, t,\delta)$, conditioned on the sequence of arm pulls, is independent of the actual observed rewards. It can be shown that 
\begin{align*}
    \sum_{t=1}^\tau U(x_t, t, \delta)-L(x_t,t,\delta)\le O(C(\delta)\sqrt{d \tau \log(\tau)})
\end{align*}
for any sequence $x_1,\dots,x_\tau$. Thus with probability at least $1-\delta$, the confidence intervals are all correct, and so we have regret at most $O(C(\delta)\sqrt{d\log(\tau)\tau})=O(d\sqrt{\log(T/\delta) \tau})$.

%% file: adv_context.tex
\section{Combining Linear Bandits with Adversarial Contexts}\label{sec:advcombiner}

In this section we consider linear contextual bandits with \emph{adversarial} rather than stochastic contexts and show that a variant of our combiner Algorithm \ref{alg:combiner} can combine instances of standard confidence-ellipsoid based bandit algorithms. Formally, let us allow each learner $\A_i$ to receive a context mapping $\psi^i_t:A\to \R^{d_i}$ that maps actions in $A$ to feature vectors in $\R^{d_i}$. The learner will use this mapping to minimize regret. We suppose there is some $J$ such that for all $t$ and $a$, the reward for action $a$ at time $t$ is $\langle \psi^J_t(a),\theta_\star\rangle$ for some $\theta_\star\in\R^{d_J}$. For simplicity of exposition, we consider the case in which $\|\theta_\star\|\le 1$, $\|\psi^i_t(a)\|\le 1$ for all $i$, $t$, and $a$, and the observed rewards satisfy $|\hat r^i_t|\le 1$ with probability 1. We consider the regret with respect to the optimal action at time $t$. That is, define
\begin{align*}
    r_{\star,t} = \max_{a\in A} \langle \psi^J_t(a), \theta_\star\rangle
\end{align*}
and the regret as:
\begin{align*}
    \sum_{t=1}^T r_{\star, t} - r_t
\end{align*}
where $r_t$ is the expected loss of the chosen action. The standard base learner for this problem is LinUCB (e.g see \cite{chu2011contextual, lattimore2018bandit}, which would obtain a base regret of $\tilde O(d_J\sqrt{T})$. We would like to combine many base learners to obtain a regret bound that depends on $T$ and $d_J$, but not $d_i$ for $i\ne J$. Specifically, we will obtain a result analogous to Theorem \ref{thm:combiner} that will enable regret $\tilde O(d_J^2\sqrt{T})$.

By mild abuse of notation, we will write $a^{i_t}_{T(i_t,t)}$ to indicate the feature under the map $\psi^i_t$ of the $T(i_t,t)$th action taken by algorithm $\A_{i_t}$. Define the regularized empirical covariance matrix $M^i_t$ as:
\begin{align*}
    M^i_t = \lambda I + \sum_{\tau=1}^t (a^i_\tau)^\top a^i_\tau
\end{align*}
and define the linear regression estimate for $\theta_\star$:
\begin{align*}
    \hat\mu^i_t &= \argmin_{\theta}\lambda \|\theta\|^2 + \sum_{\tau=1}^t (\langle a^i_\tau,\theta\rangle  - r^i_\tau)^2
\end{align*}
Now, we redefine our UCB combiner with confidence intervals. First, set 

\begin{align*}
    \beta = 4160\log\left(\frac{T\log_2(\sqrt{T/\log(T/\delta)}) + 2}{\delta}\right) + 6\lambda +16d\log(1+T/\lambda)
\end{align*}
and then define:
\begin{align*}
    U(i,t-1) &=\max_{a\in A} \langle \psi_t^i(a), \hat \mu^i_{T(i,t)-1}\rangle +\beta\sqrt{\psi_t^i(a)^\top (M^i_{T(i,t)-1})^{-1} \psi_t^i(a)} - \frac{R_i}{T}
\end{align*}
We set $\A_i$ to be the linUCB algorithm (e.g. \cite{chu2011contextual}). At time $t$, it suggests the action:
\begin{align}
    \argmax_{a\in A} \langle \psi_t^i(a), \hat \mu^i_{T(i,t)-1}\rangle +\beta\sqrt{\psi_t^i(a)^\top (M^i_{T(i,t)-1})^{-1} \psi_t^i(a)}\label{eqn:linucb}
\end{align}
We note that it is more typical to use a time-varying value for $\beta$, but in order to simplify the use of indices somewhat we present this fixed scaling. The arguments are essentially identical using standard time-varying schedules for $\beta$.

Next, we change our criterion for eliminating algorithms from the active set $I_t$. We define the quantity:
\begin{align*}
    z^i_t = \langle a^i_t, \hat \mu^i_{t-1}\rangle - \hat r^i_t - \beta \sqrt{(a^i_t)^\top (M^i_{t-1})^{-1} (a^i_t)}
\end{align*}
and we decide to eliminate an algorithm if it ever does not satisfy:
\begin{align}
    \sum_{\tau=1}^t z^i_\tau \le 2\sqrt{t\log(T/\delta)}\label{eqn:linmisspecification}
\end{align}
or if it does not satisfy
\begin{equation}
\label{eqn:misspec2}
2\sum_{\tau=1}^{t} \beta \sqrt{(a^i_\tau)^\top (M^i_{\tau-1})^{-1} (a^i_\tau)}  \le C_{i} t^{\alpha_{i}}
\end{equation}

Formally, the algorithm is stated in Algorithm \ref{alg:advlin}.

\begin{algorithm*}
\caption{Adversarial Linear Bandit Combiner}\label{alg:advlin}
\begin{algorithmic}
   \STATE{\bfseries Input: } Numbers $\{C_i\}$, $\{\alpha_i\}$, $\{R_i\}$, $T$.
   \STATE Set $T(i,0)=0$ for all $i$, set $\hat\mu^i_0=0$ for all $i$, and set $I_1 = \{1,\dots,N\}$
   \STATE Set $\beta = 4160\log\left(\frac{T\log_2(\sqrt{T/\log(T/\delta)}) + 2}{\delta}\right) + 6\lambda +16d\log(1+T/\lambda)$
   \FOR{$t=1\dots T$}
   \STATE Set $ M^i_t = \lambda I + \sum_{\tau=1}^t (a^i_\tau)^\top a^i_\tau$
   \STATE Set $U(i,t-1)=\max_{a\in A} \langle \psi_t^i(a), \hat \mu^i_{T(i,t)-1}\rangle +\beta\sqrt{\psi_t^i(a)^\top (M^i_{T(i,t)-1})^{-1} \psi_t^i(a)} - \frac{R_i}{T}$.
   \STATE Set $i_t = \argmax_{i \in I_t}U(i,t-1)$.
   \STATE Update $T(i_t,t)=T(i_t,t-1)+1$ and $T(j, t)=T(j,t-1)$ for $j\ne i_t$.
   
   \STATE Get $T(i_t,t)$th action $a_t=a^{i_t}_{T(i_t, t)}$ from $\A_{i_t}$. 
   \STATE Receive reward $\hat r_t=\hat r^{i_t}_{T(i_t,t)}$, provide reward $\hat r_t$ as feedback to $\A_{i_t}$.
   \STATE Set $\hat\mu^{i_t}_{T(i_t,t)} = \argmin_{\theta}\lambda \|\theta\|^2 + \sum_{\tau=1}^t (\langle a^{i_t}_\tau,\theta\rangle  - \hat r^{i_t}_\tau)^2$.
   
   \STATE Set $z^{i_t}_t=\langle a^{i_t}_{T(i_t, t)}, \hat \mu^{i_t}_{T(i_t,t)-1}\rangle - \hat r_t - \beta \sqrt{(a_t)^\top (M^{i_t}_{T(i_t,t)-1})^{-1} (a_t)}$
   \IF{$\sum_{\tau=1}^{T(i_t,t)} z^{i_t}_\tau >2\sqrt{T(i_t,t)\log(T/\delta)}$ or $2\sum_{\tau=1}^{T(i_t,t)} \beta \sqrt{(a^{i_t}_\tau)^\top (M^{i_t}_{\tau-1})^{-1} (a^{i_t}_\tau)} > C_{i_t} T(i_t,t)^{\alpha_{i_t}}$} 
   \STATE $I_t = I_{t-1}-\{i_t\}$.
   \ELSE
   \STATE $I_t = I_{t-1}$.
   \ENDIF
   \ENDFOR
\end{algorithmic}
\end{algorithm*}

We will use a proof similar to that of Theorem \ref{thm:combiner} to show the following:
\begin{restatable}{Theorem}{thmadvcombiner}\label{thm:advcombiner}
Suppose there is some $J\in \{1,\dots,N\}$ such that with probability at least $1-\delta$:
\begin{align}
    \sum_{\tau=1}^t r_\star - r^J_\tau &\le2\sum_{\tau=1}^t \beta \sqrt{(a^J_\tau)^\top (M^J_{\tau-1})^{-1} a^J_\tau}\nonumber\\
    &\le C_Jt^{\alpha_J}
\end{align}
and also for all $a\in A$, the expected reward $r_t(a)=\langle \psi^J_t(a), \theta_\star\rangle$ for some $\theta_\star$ with $\|\theta_\star\|\le 1$
for all $t\le T$. Further, suppose all the base learners $\A_{i}$ are linUCB algorithms using prediction (\ref{eqn:linucb}), $C_i$ and $\alpha_i$ are known, and the $R_i$ satisfy:
\begin{align*}
    R_i&\ge C_iT^{\alpha_i} \\
    R_i&\ge\sum_{k\ne i} \max\left[\frac{24 T}{R_k},\ \frac{(1-\alpha_k) \alpha_k^{\frac{\alpha_k}{1-\alpha_k}} (2C_k)^{\frac{1}{1-\alpha_k}}T^{\frac{\alpha_k}{1-\alpha_k}}}{R_k^{\frac{\alpha_k}{1-\alpha_k}}}\right]
\end{align*}
Let $r_t= \E[\hat r^{i_t}_{T(i_t,t)}]$ be the expected reward of Algorithm \ref{alg:advlin} at time $t$. Then, with probability at least $1-3\delta$, the regret satisfies:
\begin{align*}
    \sum_{t=1}^T r_\star - r_t \le 3 R_J
\end{align*}
\end{restatable}

Note that standard matrix analysis (e.g. see Lemma \ref{thm:secondmisspecification}) shows that
\begin{align*}
    \sum_{\tau=1}^T \beta \sqrt{(a^J_\tau)^\top (M^J_{\tau-1})^{-1} a^J_\tau}\le \tilde O\left(d_J\sqrt{T}\right)
\end{align*}
so that by setting $C_i = \tilde O(d_i^2)$ and $\alpha_i=1/2$ for all $i$, Theorem \ref{thm:advcombiner} yields a regret bound of $\tilde O(d_J^2\sqrt{T})$, just as we encountered in the case of stochastic contexts.

\begin{proof}
First, recall that by our linUCB update for $\A_i$ (\ref{eqn:linucb}), we have
\begin{align*}
    U(i_t, t-1) = \langle a_t, \hat \mu^{i_t}_{T(i_t,t)-1}\rangle +\beta\sqrt{a_t^\top (M^i_{T(i,t)-1})^{-1} a_t} - \frac{R_{i_t}}{T}
\end{align*}

Now, we write the regret:
\begin{align*}
    \sum_{t=1}^T r_{\star,t} -r_t&=\sum_{t=1}^T r_{\star,t} -  r^{i_t}_{T(i_t,t)}\\
    &=\sum_{i_t=J}r_{\star,t} - r^J_{T(J,t)} + \sum_{i_t\ne J} r_{\star,t} - r^{i_t}_{T(i_t,t)}
\end{align*}
First let's consider the indices for which $i_t=J$. Notice that by Corollary \ref{thm:misspecification} and Lemma \ref{thm:secondmisspecification}, a well-specified learner $\A_{J}$ is never eliminated with probability at least $1-2\delta$, and we further have by Theorem \ref{thm:linucbregret}:
\begin{align*}
    \sum_{i_t=J}r_{\star,t} - r^J_{T(J,t)} \le C_JT(J,T)^{\alpha_J}\le R_J
\end{align*}

Next, if $i_t\ne J$, we must have $U(i_t, t-1)\ge U(J, t-1)$ since $J\in I_t$ for all $t$. Then:
\begin{align*}
    r_{\star,t} - r^{i_t}_{T(i_t,t)} &\leq r_{\star,t} - U(J,t-1) + U(i_t,t-1) - r^{i_t}_{T(i_t,t)} \\
    &= r_{\star,t} - U(J,t-1) + \hat \mu^{i_t}_{T(i_t,t-1)} a_t + \beta \sqrt{a_t^\top ({{M^{i_t}_{T(i_t,t)-1}})^{-1} a_t}} - \frac{R_{i_t}}{T} - \hat r^{i_t}_{T(i_t,t)} + \hat r^{i_t}_{T(i_t,t)} - r^{i_t}_{T(i_t,t)} \\
    &= r_{\star,t} - U(J,t-1) +  z^{i_t}_{T(i_t,t)}+2\beta \sqrt{a_t^\top ({{M^{i_t}_{T(i_t,t)-1}})^{-1} a_t}} - \frac{R_{i_t}}{T}  +  \hat r^{i_t}_{T(i_t,t)} - r^{i_t}_{T(i_t,t)}\\ 
\end{align*}
where the last inequality follows from the definition of $z^i_t$.

Now, $r_{\star,t} - U(J,t-1) - \frac{R_j}{T} = r_{\star,t} - \max_{a\in A} \left[\hat \mu^J_{T(J,t-1)} a + \beta\sqrt{a^T ({{M^J_{T(J,t)-1}})^{-1} a}}\right] \leq 0$, with probability $1-\delta$ by Corollary \ref{thm:misspecification}.
Therefore, conditioned on this $1-\delta$ event, we have
\begin{align*}
    \sum_{i_t\ne J} r_{\star,t} - r^{i_t}_{T(i_t,t)} &= \sum_{i_\ne J} \sum_{t|i_t = i } r_{\star,t} - r^{i_t}_{T(i_t,t)}\\
    &\le \sum_{i\ne J} \sum_{t|i_t = i } \frac{R_J}{T} +z^{i_t}_{T(i_t,t)}+ 2\beta \sqrt{a_t^\top ({{M^{i}_{T(i,t)-1}})^{-1} a_t}} - \frac{R_{i}}{T}  + \hat r^{i}_{T(i,t)} - r^{i}_{T(i,t)}\\
    &\le R_J  + \sum_{i\ne J} \left[ -\frac{T(i,T) R_{i}}{T} + 2\sqrt{T(i,T)\log(T/\delta)}+\sum_{t|i_t = i} \hat r^{i}_{T(i,t)} - r^{i}_{T(i,t)} \right.\\
    &\quad\quad\left.+ \sum_{t|i_t = i} 2\beta \sqrt{a_t^\top ({{M^{i}_{T(i,t)-1}})^{-1} a_t}} \right] \\
    &\le R_J  + \sum_{i\ne J} \left[ -\frac{T(i,T) R_{i}}{T} + 2\sqrt{T(i,T)\log(T/\delta)}+\sqrt{8\log(T^3N/\delta)T(i,T)}\right.\\
    &\quad\quad\left.+ \sum_{t|i_t = i} 2\beta \sqrt{a_t^\top ({{M^{i}_{T(i,t)-1}})^{-1} a_t}} \right]\\ 
    &\le R_J  + \sum_{i\ne J} \left[ -\frac{T(i,T) R_{i}}{T} +\sqrt{24\log(T^3N/\delta)T(i,T)}\right.\\
    &\quad\quad\left.+ \sum_{t|i_t = i} 2\beta \sqrt{a_t^\top ({{M^{i}_{T(i,t)-1}})^{-1} a_t}} \right]
\end{align*}
where the third inequality is from equation (\ref{eqn:linmisspecification}), and the last inequality follows from Lemma \ref{thm:martingale}

Now, $\sum_{t|i_t = i} 2\beta \sqrt{a_t^\top ({{M^{i}_{T(i,t)-1}})^{-1} a_t}} \leq C_{i} T(i,T)^{\alpha_{i}}$, where the inequality follows from the second condition for being in $I_t$ (\ref{eqn:misspec2}).

Therefore, 
\begin{align*}
    \sum_{i_t\ne J} r_{\star,t} - r^{i_t}_{T(i_t,t)} &\le R_J  + \sum_{i_\ne J} \left[ -\frac{T(i,T) R_{i}}{T} + \sqrt{24\log(T^3N/\delta)T(i,T)} +  C_{i} T(i,T)^{\alpha_{i}} \right]\\
    &\le R_J + \sum_{i\ne J}\max\left\{ \sup_Z  \left[ -\frac{Z R_{i}}{T} + 2\sqrt{24\log(T^3N/\delta)Z}\right],\sup_Z \left[ -\frac{Z R_{i}}{T} + 2C_{i} T(i,T)^{\alpha_{i}}\right] \right\}
    \intertext{By Lemma \ref{thm:alphabound}:}
    &\le R_J + \sum_{i\ne J} \max\left[\frac{24 T}{R_i},\ \frac{(1-\alpha_i) \alpha_i^{\frac{\alpha_i}{1-\alpha_i}} (2C_i)^{\frac{1}{1-\alpha_i}}T^{\frac{\alpha_i}{1-\alpha_i}}}{R_i^{\frac{\alpha_i}{1-\alpha_i}}}\right]\\
    &\le 2R_J
\end{align*}    
This component of the regret due to ${i_t\ne J}$ is now in the same form as in the proof of the Theorem \ref{thm:combiner} and can be bounded by $O(R_J)$ by an almost identical analysis. 

And so, putting all this together we have with probability at least $1-3\delta$, $\sum_{t=1}^T r_{\star,t} -r_t \le 3R_J$. Therefore with probability at least $1-3\delta$, the regret is bounded by $3 R_j$.
\end{proof}

%% file: adv_linucb.tex
\section{Adversarial linear contextual bandits analysis}\label{sec:advlinucb}

In this section we analyze the standard confidence ellipsoid technique for linear contextual bandits in the case of adversarial contexts. We do not claim these results are novel (essentially similar and tighter analysis can be found in \cite{chu2011contextual,abbasi2011improved}). We include this section for completeness and to complement the analysis in Section \ref{sec:advcombiner}. Our proof technique is however slightly different from these works - it is a variant on the online-to-confidence set conversion idea of \cite{abbasi2012online} that might have some indepenedent interest.

In the linear contextual bandit problem, in each round the adversary reveals a mapping $\psi_t$ from the set of actions $A$ to $\R^d$. There is an unknown fixed $\theta_\star \in \R^d$ such that the expected reward of action $a$ is always $\langle \psi_t(a),\theta_\star\rangle$. We assume $\|\psi_t(a)\|\le 1$ for all $a$ and $t$ and $\|\theta_\star\|\le 1$. Let $a_1,\dots,a_T$ be the actions chosen by the algorithm over $T$ rounds. Then the regret is defined as:
\begin{align*}
    \sum_{t=1}^T \max_{a\in A} \langle \psi_t(a)-\psi_t(a_t),\theta_\star \rangle
\end{align*}
By abuse of notation, we will identify actions with their associated vectors in order to write $a^\star_t$ and $a_t$ instead of $\argmax_{a} \psi_t(a)$ and $\psi_t(a_t)$. Thus, the regret is
\begin{align*}
    \sum_{t=1}^T \langle a^\star_t- a_t,\theta_\star\rangle
\end{align*}

After playing action $a_t$, we observe $y_t\in [-1,1]$\footnote{it is also possible to consider subgaussian $y_t$, but we stick with bounded values here for simplicity} with $\E[y_t|a_t] = \langle a_t,\theta_\star\rangle$.

Define $A_t=\lambda I +\sum_{\tau=1}^t a_\tau^\top a_\tau$, where $\lambda >0$ is some scalar parameter, and $I$ is the $d\times d$ identity matrix. Define
\begin{align*}
    \hat \mu_t = \argmin_{\|\theta\|\le 1} \frac{\lambda}{2} \|\theta\|^2 + \frac{1}{2}\sum_{\tau=1}^{t-1} (\langle a_t,\theta\rangle - y_t)^2
\end{align*}

In round $t$, we choose
\begin{align}
    a_t = \argmax_{a} \langle a, \mu_{t-1}\rangle + \beta_{t-1} \sqrt{a^\top A_{t-1}^{-1}a}\label{eqn:achoice}
\end{align}
for some appropriate scaling factor $\beta_t$ we will choose shortly. We claim the following:
\begin{Lemma}\label{thm:linucbconfidence}
With probability at least $1-\delta$, 
\begin{align*}
    (\theta_\star - \hat \mu_t)^\top A_t (\theta_\star - \hat \mu_t)\le 4160\log\left(\frac{\log_2(\sqrt{T/\log(1/\delta)}) + 2}{\delta}\right) + 6\lambda +16d\log(1+T/\lambda)
\end{align*}
\end{Lemma}
\begin{proof}
Observe that $\hat\mu_t$ is the output of the follow-the-regularized-leader algorithm (e.g. \cite{mcmahan2017survey}) on the losses $\ell_t(\mu) = \frac{1}{2}(\langle a_t, \mu\rangle -y_t)^2$ with constant regularizer $\frac{\lambda}{2} \|\mu\|^2$ restricted to vectors of norm at most 1. Thus, by \cite{mcmahan2017survey} Theorem 2, we have
\begin{align*}
    \sum_{t=1}^T \frac{1}{2}(\langle a_t, \mu_t\rangle -y_t)^2 -\frac{1}{2}(\langle a_t, \theta_\star\rangle -y_t)^2& \le \frac{\lambda}{2}\|\theta_\star\|^2+\sum_{t=1}^T (\langle a_t, \mu_t\rangle -y_t)^2\frac{a_t^\top A_t^{-1}a_t}{2}\\
    &\le \frac{\lambda}{2} + 4\sum_{t=1}^T\frac{a_t^\top A_t^{-1}a_t}{2}\\
    &\le \frac{\lambda}{2} + 4d\log(1+T/\lambda)
\end{align*}
where the last line is by \cite{hazan2007logarithmic} Lemma 11.

Now, let $z_t = y_t - \langle a_t,\theta_\star\rangle$. Observe that $\E[z_t]=0$ and $z_t$ is 2-subgaussian. Therefore:
\begin{align*}
    &\sum_{t=1}^T \frac{1}{2}(\langle a_t, \hat \mu_t\rangle -y_t)^2 -\frac{1}{2}(\langle a_t, \theta_\star\rangle -y_t)^2=\sum_{t=1}^T \frac{1}{2}\langle a_t,\hat \mu_t - \theta_\star\rangle^2 -z_t \langle a_t,\hat \mu_t-\theta_\star\rangle
    \intertext{Now by Lemma \ref{thm:azumahoeffding}, with probability at least $1-\delta$ we have:}
    &\quad\ge \sum_{t=1}^T \frac{1}{2}\langle a_t,\hat \mu_t- \theta_\star\rangle^2 - 16\sqrt{\log\left(\frac{\log_2(\sqrt{T/\log(1/\delta)}) + 2}{\delta}\right)\sum_{t=1}^T \langle a_t,\hat \mu_t - \theta_\star\rangle^2}\\
    &\quad\quad\quad-16\log\left(\frac{\log_2(\sqrt{T/\log(1/\delta)}) + 2}{\delta}\right)
    \intertext{Next, use the identity $A-B\sqrt{A}\ge \frac{A}{2}-2B^2$:}
    &\quad\ge \frac{1}{4}\sum_{t=1}^T \langle a_t,\hat \mu_t- \theta_\star\rangle^2-1040\log\left(\frac{\log_2(\sqrt{T/\log(1/\delta)}) + 2}{\delta}\right)
\end{align*}
Thus with probability at least $1-\delta$:
\begin{align*}
    \sum_{t=1}^T \frac{1}{4}\langle a_t,\hat \mu_t- \theta_\star\rangle^2&\le 1040\log\left(\frac{\log_2(\sqrt{T/\log(1/\delta)}) + 2}{\delta}\right) + \frac{\lambda}{2} + 4d\log(1+T/\lambda)\\
    \lambda \|\hat \mu_t-\theta_\star\|^2 + \sum_{t=1}^T \langle a_t,\hat \mu_t-\theta_\star\rangle^2&\le 4160\log\left(\frac{\log_2(\sqrt{T/\log(1/\delta)}) + 2}{\delta}\right) + 6\lambda +16d\log(1+T/\lambda)\\
    (\hat \mu_t-\theta_\star)^\top A_t(\hat \mu_t-\theta_\star)&\le 4160\log\left(\frac{\log_2(\sqrt{T/\log(1/\delta)}) + 2}{\delta}\right) + 6\lambda +16d\log(1+T/\lambda)
\end{align*}
\end{proof}

This Lemma has an important corollary:
\begin{Corollary}\label{thm:misspecification}
Define
\begin{align*}
    \beta_t^2 = 4160\log\left(\frac{T\log_2(\sqrt{T/\log(T/\delta)}) + 2}{\delta}\right) + 6\lambda +16d\log(1+T/\lambda)
\end{align*}
and
\begin{align*}
    z_t = \langle a_t, \hat \mu_t\rangle - y_t - \beta_t \sqrt{a_t^\top A_t a_t}
\end{align*}
Then
\begin{align*}
    \E[z_t]\le 0
\end{align*}
and with probability at least $1-2\delta$
\begin{align*}
\sum_{t=1}^\tau z_t \le 2\sqrt{\tau \log(T/\delta)}
\end{align*}
and also
\begin{align*}
    \max_{a} \langle a, \theta_\star\rangle - \max_{a}\langle a, \hat \mu_{t-1}\rangle + \beta_t\sqrt{a^\top A_{t-1}^{-1} a}\le 0
\end{align*}
for all $\tau \le T$
\end{Corollary}
\begin{proof}
First, we claim that with probability at least $1-\delta$, for any $a$:
\begin{align*}
    \langle a, \hat \mu_{t-1} - \theta_\star \rangle - \beta_t \sqrt{a^\top A_{t-1} a}\le 0
\end{align*}
for all $t$. To see this, notice that by Lemma \ref{thm:linucbconfidence}, we have with probability $1-\delta$ for all $t$.
\begin{align*}
    (\theta_\star - \hat \mu_t)^\top A_t (\theta_\star - \hat \mu_t)\le \beta_t^2
\end{align*}
Therefore
\begin{align*}
    \langle a, \theta_\star\rangle &=\langle a,\hat \mu_{t-1}\rangle + \langle a,\theta_\star -\hat\mu_{t-1}\rangle\\
    &\le \langle a,\hat\mu_{t-1}\rangle + \beta_t \sqrt{a^\top A_{t-1}^{-1} a}
\end{align*}
This shows the last claim of the Corollary.
Now let $x_t = \langle a_t,\theta_\star\rangle-y_t$. Notice that $\E[x_t]=0$ and $x_t\in[-2,2]$. Thus we have with probability at least $1-\delta$,
\begin{align*}
    \sum_{t=1}^\tau x_t \le 2\sqrt{\tau \log(T/\delta)}
\end{align*}
for all $\tau$. Combining all together proves the Corollary.
\end{proof}

Next, we need to bound the actual regret of this algorithm:
\begin{Theorem}\label{thm:linucbregret}
Set 
\begin{align*}
    \beta_t^2 = 4160\log\left(\frac{T\log_2(\sqrt{T/\log(T/\delta)}) + 2}{\delta}\right) + 6\lambda +16d\log(1+T/\lambda)
\end{align*}
Then, if $\lambda \ge 2$, with probability at least $1-\delta$:
\begin{align*}
    \sum_{t=1}^T \langle a^\star_t-a_t,\theta_\star\rangle&\le 1\sum_{t=1}^T \beta_t\sqrt{a_t^\top A_{t-1}^{-1} a_t} \\
    &\le \beta_T\sqrt{dT\log(1+2T/\lambda)}\\
    &=\tilde O\left(d\sqrt{T}\right)
\end{align*}
\end{Theorem}
\begin{proof}
By Lemma \ref{thm:linucbconfidence}, we have
\begin{align*}
    (\theta_\star - \hat \mu_t)^\top A_t (\theta_\star - \hat \mu_t)\le \beta_t^2
\end{align*}
for all $t$ with probability at least $1-\delta$. Therefore conditioned on this $1-\delta$ probability event, for any $a$:
\begin{align*}
    \langle a,\theta_\star\rangle&=\langle a\hat\mu_{t-1}\rangle + \langle a,\theta_\star-\hat\mu_{t-1}\rangle\\
    &\le \langle a, \hat\mu_{t-1} \rangle + \beta_{t-1} \sqrt{a^\top A_{t-1}^{-1}a}
    \intertext{so by definition of $a_t$:}
    &\le \langle a_t,\hat \mu_{t-1}\rangle + \beta_{t-1}\sqrt{a_t^\top A_{t-1}^{-1} a_t}
\end{align*}
Similarly, we have
\begin{align*}
    \langle a_t,\theta_\star\rangle &= \langle a_t,\hat\mu_{t-1}\rangle + \langle a_t,\theta_\star-\hat\mu_{t-1}\rangle\\
    &\ge \langle a_t,\hat\mu_{t-1}\rangle - \beta_t \sqrt{a_t^\top A_{t-1}^{-1}a_t}
\end{align*}
Therefore:
\begin{align*}
    \langle a^\star_t - a_t,\theta_\star\rangle &\le 2\beta_t\sqrt{a_t^\top A_{t-1}^{-1} a_t}
\end{align*}
Thus by Cauchy-Schwarz and the monotonicity of $\beta_t$:
\begin{align*}
    \sum_{t=1}^T \langle a^\star_t-a_t,\theta_\star\rangle&\le 2\beta_T\sqrt{T\sum_{t=1}^T a_t^\top A_{t-1}^{-1} a_t}
\end{align*}
Now define $\hat A_t= \frac{\lambda}{2}I + \sum_{\tau=1}^t a_\tau^\top a_\tau$. Since $\|a_t\|\le 1$, if we set $\lambda\ge 2$, we must have
\begin{align*}
    a_t^\top A_{t-1}^{-1}a_t\le a_t^\top \hat A_t^{-1}a_t
\end{align*}
Now, again by \cite{hazan2007logarithmic} Lemma 11, we have:
\begin{align*}
    \sum_{t=1}^T a_t^\top \hat A_t^{-1} a_t&\le d\log(1+2T/\lambda)
\end{align*}
which concludes the Theorem.
\end{proof}

Finally, one more check:
\begin{Lemma}\label{thm:secondmisspecification}
Suppose $\lambda\ge 2$. Then for all $t$,
\begin{align*}
\sum_{\tau=1}^t \beta_\tau \sqrt{a_t^\top A_{t-1}^{-1} a_t}\le \beta_t \sqrt{d t\log(1+2t/\lambda)}
\end{align*}
\end{Lemma}
\begin{proof}
Since $\|a_t\|\le 1$, we have
\begin{align*}
    a_t^\top A_{t-1}^{-1} a_t\le a_t^\top \hat A_{t-1}^{-1} a_t
\end{align*}
where
\begin{align*}
    \hat A_t = \frac{\lambda}{2}I + \sum_{\tau=1}^t a_\tau^\top a_\tau
\end{align*}
Then by Cauchy-Schwarz and the monotonicity of $\beta_t$:
\begin{align*}
    \sum_{\tau=1}^t \beta_\tau \sqrt{a_t^\top A_{t-1}^{-1} a_t}&\le \beta_t\sqrt{t} \sqrt{\sum_{\tau=1}^t a_\tau^\top A_{\tau-1}^{-1} a_\tau}\\
    &\le \beta_t\sqrt{t} \sqrt{\sum_{\tau=1}^t a_\tau^\top \hat A_{\tau}^{-1} a_\tau}
    \intertext{now use \cite{hazan2007logarithmic} Lemma 11 once again:}
    &\le \beta_t\sqrt{dt\log(1+2t/\lambda)}
\end{align*}
\end{proof}

\begin{Lemma}\label{thm:azumahoeffding}
Suppose $x_1,\dots,X_T$ are arbitrary random variables with $|x_t|\le 1$ almost surely. Suppose $z_1,\dots,z_T$ are such that $\E[z_t|x_1,\dots,x_t]=0$ and $z_t$ is 1-subgaussian given $x_1,\dots,x_t$. Then with probability at least $1-\delta$:
\begin{align*}
    \sum_{t=1}^T z_t x_t \ge -4\sqrt{\log\left(\frac{\log_2(\sqrt{T/\log(1/\delta)}) + 2)}{\delta}\right)\sum_{t=1}^T x_T^2 }+4\log\left(\frac{\log_2(\sqrt{T/\log(1/\delta)}) + 2)}{\delta}\right)
\end{align*}
\end{Lemma}
\begin{proof}
The proof should follow from standard inequalities. Here we just check that $x_t$ being random does not cause a significant problem... Define $S_\tau = \sum_{t=1}^\tau z_t x_t$. Note that $S_1,S_2,\dots$ is a martingale. For any $\eta$, we have:
\begin{align*}
    P[S_T\ge \epsilon] &\le \exp^{-\eta\epsilon} \E[\exp(\eta S_Tt)]\\
    &=\exp^{-\eta\epsilon}\E[\exp(\eta S_{T-1})\exp(\eta z_Tx_T)]\\
    &= \exp^{-\eta\epsilon}\E[\exp(\eta S_{T-1})\E[\exp(\eta z_Tx_T)|x_1,\dots,x_T]]\\
    &\le \exp(-\eta\epsilon) \E[ \exp(\eta S_{t-1})\exp(\eta^2 x_T^2/2)]
    \intertext{repeating for $T$ steps:}
    &\le \exp(-\eta\epsilon+\eta^2\sum_{t=1}^T x_T^2/2)
\end{align*}
Therefore with probability at least $1-\delta$,
\begin{align*}
    S_T \le \frac{\log(1/\delta)}{\eta}+\eta\sum_{t=1}^T x_T^2
\end{align*}
Then, considering $\eta=2^k$ for $-\lceil \log_2(\sqrt{T/\log(1/\delta)})\rceil \le k\le 0$, we have that with probability at least $1-\delta$,
\begin{align*}
    S_T &\le \min_{-\lceil \log_2(\sqrt{T/\log(1/\delta)})\rceil \le k\le 0} \frac{\log\left(\frac{\log_2(\sqrt{T/\log(1/\delta)}) + 2}{\delta}\right)}{2^k}+2^k\sum_{t=1}^T x_T^2\\
    &\le 2\inf_{\eta\in[\sqrt{\log(1/\delta)/T},1]}\frac{\log\left(\frac{\log_2(\sqrt{T/\log(1/\delta)}) + 2}{\delta}\right)}{\eta}+\eta\sum_{t=1}^T x_T^2\\
    &\le 4\sqrt{\log\left(\frac{\log_2(\sqrt{T/\log(1/\delta)}) + 2}{\delta}\right)\sum_{t=1}^T x_T^2 }+4\log\left(\frac{\log_2(\sqrt{T/\log(1/\delta)}) + 2}{\delta}\right)
\end{align*}
\end{proof}

%% file: arxiv.bbl
\begin{thebibliography}{10}

\bibitem{langford2008epoch}
John Langford and Tong Zhang.
\newblock The epoch-greedy algorithm for multi-armed bandits with side
  information.
\newblock In {\em Advances in neural information processing systems}, pages
  817--824, 2008.

\bibitem{beygelzimer2011contextual}
Alina Beygelzimer, John Langford, Lihong Li, Lev Reyzin, and Robert Schapire.
\newblock Contextual bandit algorithms with supervised learning guarantees.
\newblock In {\em Proceedings of the Fourteenth International Conference on
  Artificial Intelligence and Statistics}, pages 19--26, 2011.

\bibitem{lai1985asymptotically}
Tze~Leung Lai and Herbert Robbins.
\newblock Asymptotically efficient adaptive allocation rules.
\newblock {\em Advances in applied mathematics}, 6(1):4--22, 1985.

\bibitem{auer2002finite}
Peter Auer, Nicolo Cesa-Bianchi, and Paul Fischer.
\newblock Finite-time analysis of the multiarmed bandit problem.
\newblock {\em Machine learning}, 47(2-3):235--256, 2002.

\bibitem{lattimore2018bandit}
Tor Lattimore and Csaba Szepesv{\'a}ri.
\newblock Bandit algorithms.
\newblock {\em preprint}, page~28, 2018.

\bibitem{tewari2017ads}
Ambuj Tewari and Susan~A Murphy.
\newblock From ads to interventions: Contextual bandits in mobile health.
\newblock In {\em Mobile Health}, pages 495--517. Springer, 2017.

\bibitem{agarwal2014taming}
Alekh Agarwal, Daniel Hsu, Satyen Kale, John Langford, Lihong Li, and Robert
  Schapire.
\newblock Taming the monster: A fast and simple algorithm for contextual
  bandits.
\newblock In {\em International Conference on Machine Learning}, pages
  1638--1646, 2014.

\bibitem{agarwal2017corralling}
Alekh Agarwal, Haipeng Luo, Behnam Neyshabur, and Robert~E Schapire.
\newblock Corralling a band of bandit algorithms.
\newblock In {\em Conference on Learning Theory}, pages 12--38, 2017.

\bibitem{pacchiano2020model}
Aldo Pacchiano, My~Phan, Yasin Abbasi-Yadkori, Anup Rao, Julian Zimmert, Tor
  Lattimore, and Csaba Szepesvari.
\newblock Model selection in contextual stochastic bandit problems.
\newblock {\em arXiv preprint arXiv:2003.01704}, 2020.

\bibitem{bubeck2012best}
S{\'e}bastien Bubeck and Aleksandrs Slivkins.
\newblock The best of both worlds: Stochastic and adversarial bandits.
\newblock In {\em Conference on Learning Theory}, pages 42--1, 2012.

\bibitem{wei2018more}
Chen-Yu Wei and Haipeng Luo.
\newblock More adaptive algorithms for adversarial bandits.
\newblock {\em Proceedings of Machine Learning Research}, 75, 2018.

\bibitem{audibert2009exploration}
Jean-Yves Audibert, R{\'e}mi Munos, and Csaba Szepesv{\'a}ri.
\newblock Exploration--exploitation tradeoff using variance estimates in
  multi-armed bandits.
\newblock {\em Theoretical Computer Science}, 410(19):1876--1902, 2009.

\bibitem{ghosh2020}
Avishek Ghosh, Abishek Sankararaman, and Kannan Ramchandran.
\newblock Problem-complexity adaptive model selection for stochastic linear
  bandits.
\newblock {\em arXiv preprint arXiv:2006.02612}, 2020.

\bibitem{arora2020corralling}
Raman Arora, Teodor~V Marinov, and Mehryar Mohri.
\newblock Corralling stochastic bandit algorithms.
\newblock {\em arXiv preprint arXiv:2006.09255}, 2020.

\bibitem{yadkori2020}
Yasin Abbasi-Yadkori, Aldo Pacchiano, and My~Phan.
\newblock Regret balancing for bandit and rl model selection.
\newblock {\em arXiv preprint arXiv:2006.05491}, 2020.

\bibitem{chu2011contextual}
Wei Chu, Lihong Li, Lev Reyzin, and Robert Schapire.
\newblock Contextual bandits with linear payoff functions.
\newblock In {\em Proceedings of the Fourteenth International Conference on
  Artificial Intelligence and Statistics}, pages 208--214, 2011.

\bibitem{lattimore2015pareto}
Tor Lattimore.
\newblock The pareto regret frontier for bandits.
\newblock In {\em Advances in Neural Information Processing Systems}, pages
  208--216, 2015.

\bibitem{foster2019model}
Dylan~J Foster, Akshay Krishnamurthy, and Haipeng Luo.
\newblock Model selection for contextual bandits.
\newblock In {\em Advances in Neural Information Processing Systems}, pages
  14714--14725, 2019.

\bibitem{chatterji2019osom}
Niladri~S Chatterji, Vidya Muthukumar, and Peter~L Bartlett.
\newblock Osom: A simultaneously optimal algorithm for multi-armed and linear
  contextual bandits.
\newblock {\em arXiv preprint arXiv:1905.10040}, 2019.

\bibitem{abbasi2011improved}
Yasin Abbasi-Yadkori, D{\'a}vid P{\'a}l, and Csaba Szepesv{\'a}ri.
\newblock Improved algorithms for linear stochastic bandits.
\newblock In {\em Advances in Neural Information Processing Systems}, pages
  2312--2320, 2011.

\bibitem{Dani2008StochasticLO}
Varsha Dani, Thomas~P. Hayes, and Sham~M. Kakade.
\newblock Stochastic linear optimization under bandit feedback.
\newblock In {\em COLT}, 2008.

\bibitem{BubeckC12}
S{\'{e}}bastien Bubeck and Nicol{\`{o}} Cesa{-}Bianchi.
\newblock Regret analysis of stochastic and nonstochastic multi-armed bandit
  problems.
\newblock {\em Foundations and Trends in Machine Learning}, 5(1):1--122, 2012.

\bibitem{Slivkins19}
Aleksandrs Slivkins.
\newblock Introduction to multi-armed bandits.
\newblock {\em Foundations and Trends in Machine Learning}, 12(1-2):1--286,
  2019.

\bibitem{abbasi2012online}
Yasin Abbasi-Yadkori, David Pal, and Csaba Szepesvari.
\newblock Online-to-confidence-set conversions and application to sparse
  stochastic bandits.
\newblock In {\em Artificial Intelligence and Statistics}, pages 1--9, 2012.

\bibitem{mcmahan2017survey}
H~Brendan McMahan.
\newblock A survey of algorithms and analysis for adaptive online learning.
\newblock {\em The Journal of Machine Learning Research}, 18(1):3117--3166,
  2017.

\bibitem{hazan2007logarithmic}
Elad Hazan, Amit Agarwal, and Satyen Kale.
\newblock Logarithmic regret algorithms for online convex optimization.
\newblock {\em Machine Learning}, 69(2-3):169--192, 2007.

\end{thebibliography}
